\documentclass[lettersize,journal]{IEEEtran}

\usepackage{xcolor}
\usepackage{comment}
\usepackage{verbatim}

\usepackage{cite}
\usepackage{url}

\usepackage{textcomp}
\usepackage{stfloats}
\usepackage[mathcal]{euscript}

\usepackage[cmex10]{amsmath}
\usepackage{amssymb} 
\usepackage{amsfonts}

\usepackage{amsthm} 
  
\usepackage{dsfont} 
\usepackage{mathabx} 

\usepackage{graphicx}
\usepackage{epsfig} 
\usepackage[caption=false,font=normalsize,labelfont=sf,textfont=sf]{subfig}

\usepackage{array}
\usepackage{multirow}
\usepackage{enumerate}

\usepackage[ruled, vlined, linesnumbered]{algorithm2e}

\newtheoremstyle{plain}
	  {}
	  {}
	  {\itshape}
	  {}
	  {\bfseries}
	  {}
	  {5pt plus 1pt minus 1pt}
	  {}

\newtheoremstyle{definition}
  	  {}
	  {}
	  {\normalfont}
	  {}
	  {\bfseries}
	  {}
	  {5pt plus 1pt minus 1pt}
	  {}
  
\theoremstyle{plain}

\newtheorem{lemma}{Lemma}
\newtheorem{proposition}{Proposition}
\newtheorem{corollary}{Corollary}

\theoremstyle{definition}
\newtheorem{definition}{Definition}

\newcommand{\refeq}[1]			{(\ref{#1})} 
\newcommand{\reffig}[1]			{Fig. \ref{#1}} 
\newcommand{\refsec}[1]			{Section \ref{#1}}

\newcommand{\refprop}[1]		{Proposition \ref{#1}}

\newcommand{\refdef}[1]			{Definition \ref{#1}}

\newcommand{\reffn}[1] 		    {\textsuperscript{\ref{#1}}}

\theoremstyle{plain}

\newcommand{\R}  	{\mathbb{R}} 
\newcommand{\dimspace} 	{d}
\newcommand{\radius} 	{\rho}

\newcommand{\conv}      {\mathrm{conv}} 

\newcommand{\freespace}	{\mathcal{F}} 
\newcommand{\workspace}	{\mathcal{W}} 
\newcommand{\obstspace}	{\mathcal{O}} 
\newcommand{\ball}      {\mathrm{B}} 

\newcommand{\clearance} {\epsilon} 

\newcommand{\pos} 		{\vect{x}} 			
\newcommand{\dimpos}		{\dimspace_{\pos}}   
\newcommand{\nonpos}    	{\vect{v}} 
\newcommand{\dimnonpos} 	{\dimspace_{\nonpos}} 

\newcommand{\orient}	    {\theta}			
\newcommand{\linvel}     {v}  
\newcommand{\angvel}     {\omega} 

\newcommand{\goal}		{\vect{x}^*} 

\newcommand{\phdctrl}{{\hat{\vect{u}}}}

\newcommand{\state}		{\mat{x}} 			
\newcommand{\order}     {n} 				
\newcommand{\gain}   	{\kappa} 			
\newcommand{\dyn}		{\vect{f}} 
\newcommand{\ctrl}      {\vect{u}} 			
\newcommand{\navctrl}  	{\overline{\vect{u}}} 
\newcommand{\pathctrl}   {\vect{u}} 
\newcommand{\dimctrl}   {{\dimspace_{\ctrl}}}  
\newcommand{\lyapmat}	{\mat{P}} 
\newcommand{\decaymat}  {\mat{D}} 
\newcommand{\phdmat}    {\mat{C}}

\newcommand{\refpath}	{\vect{p}} 
\newcommand{\pathparam}  {s} 
\newcommand{\maxpathparam}{{s_{\max}}} 
\newcommand{\minpathparam}{{s_{\min}}} 
\newcommand{\pathparamrate} {v_{s}} 
\newcommand{\pathdyn} 	{g} 

\newcommand{\motionset}		{\mathcal{M}} 
\newcommand{\setradius}		{r} 
\newcommand{\motionelp}	{\mathcal{ME}} 
\newcommand{\motionspx}{\mathcal{MS}} 

\newcommand{\safetylevel}{\sigma} 

\newcommand{\phdroot}			{\lambda}
\newcommand{\phdroots}			{\boldsymbol{\lambda}}
\newcommand{\phdgain}			{\kappa}

\newcommand{\PDM} 	{S_{++}} 
\newcommand{\PSDM} 	{S_{+}} 

\newcommand{\dist} {\mathrm{dist}} 

\newcommand{\elp}		{\mathcal{E}} 
\newcommand{\elpctr}	{\vect{c}} 
\newcommand{\elpmat} 	{\mat{\Sigma}} 
\newcommand{\elprad}	{\rho} 

\let\originalleft\left
\let\originalright\right
\renewcommand{\left}{\mathopen{}\mathclose\bgroup\originalleft}
\renewcommand{\right}{\aftergroup\egroup\originalright}
	
\newcommand{\plist}[1] 	{\left(#1\right)} 
\newcommand{\blist}[1]	{\left[ #1 \right]} 
\newcommand{\clist}[1]	{\left\{#1\right\}} 

\newcommand{\vect}[1]   {\mathrm{#1}}
\newcommand{\mat}[1]    {\mathbf{#1}}
\newcommand{\tr}[1] {{#1}^{\mathrm{T}}} 
\newcommand{\norm}[1]  {\|#1\|}
\newcommand{\absval}[1]{\left|#1 \right|} 
\newcommand{\diag} {\mathrm{diag}} 

\newcommand{\diff} {\mathrm{d}}

\begin{document}

\title{Time Governors for Safe Path-Following Control}


\author{\"{O}m\"{u}r Arslan
\thanks{The author is with the Department of Mechanical Engineering, Eindhoven University of Technology, P.O. Box 513, 5600 MB Eindhoven, The Netherlands. The author is also affiliated with the Eindhoven AI Systems Institute. Emails:   o.arslan@tue.nl}%
}


\markboth{Technical Report, December~2022}%
{Arslan: Time Governors for Safe Path-Following Control}


\maketitle

\begin{abstract}
Safe and smooth robot motion around obstacles is an essential skill for autonomous robots, especially when operating around people and other robots. 
Conventionally, due to real-time operation requirements and onboard computation limitations, many robot motion planning and control methods follow a two-step approach: first construct a (e.g., piecewise linear) collision-free reference path for a simplified robot model, and then execute the reference plan via path-following control for a more accurate and complex robot model.    
A challenge of such a decoupled robot motion planning and control method for highly dynamic robotic systems is ensuring the safety of path-following control as well as  the successful completion of the reference plan.
In this paper, we introduce a novel dynamical systems approach for online closed-loop time parametrization, called \emph{a time governor}, of a reference path for provably correct and safe path-following control based on feedback motion prediction, where the safety of robot motion under path-following control is continuously monitored using predicted robot motion. 
After introducing the general framework of time governors for safe path following, we present an example application for the fully actuated high-order  robot dynamics using proportional-and-higher-order-derivative (PhD) path-following control whose feedback motion prediction is performed by Lyapunov ellipsoids and Vandemonde simplexes.
In numerical simulations, we investigate the role of reference position and velocity feedback, and motion prediction on path-following~performance and robot motion.
\end{abstract}


\section{Introduction}
\label{sec.Introduction}

Achieving truly reliable and dependable robots that autonomously operate around people and other robots requires provably correct, safe, and smooth robot motion planning and control methods.
Kinodynamic planning of dynamically feasible and safe robot motion is known to be computationally hard for many robotic systems \cite{donald_etal_JACM1993, lavalle_kuffner_IJRR2001}, because determining the safety of highly dynamic robot motion is difficult \cite{fraichard_asama_AR2004}.
As a result, due to real-time operation requirements and onboard computation limitations, the state-of-the-art robot motion planning and control methods often follow a two-step path planning and following approach that decouples high-level path planning and low-level motion control as follows \cite{choset_etal_PrinciplesOfRobotMotion2005, siciliano_etal_RoboticsModellingPlanningControl2009, lynch_park_ModernRobotics2017}: (i) first find a (e.g., piecewise linear/smooth) collision-free reference path for a simplified robot model, (ii) and then realize the reference plan as accurately as possible by path-following control of the actual complex robot dynamics.
However, systematically ensuring both the safety and correctness of such a decoupled robot motion planning and control approach is still an open research challenge.

In this paper, for a generic family of high-level reference path planners \cite{lavalle_PlanningAlgorithms2006} and low-level path-following control policies \cite{siciliano_etal_RoboticsModellingPlanningControl2009}, we present a new dynamical systems approach for online closed-loop time parametrization, called \emph{a time governor}, of the reference path for provably correct and safe path-following control based on feedback motion prediction and the safety assessment of predicted robot motion, as illustrated in \reffig{fig.TimeGovernor}.
After presenting our general time governor framework for safe path-following control, we provide an example application for the  fully actuated high-order robot dynamics using  proportional-and-higher-order-derivative (PhD) path following control whose feedback motion prediction is performed by Lyapunov ellipsoids and Vandemonde simplexes \cite{isleyen_vandewouw_arslan_RAL2022, arslan_isleyen_2022}.

\begin{figure}
\centering
\includegraphics[width=\columnwidth]{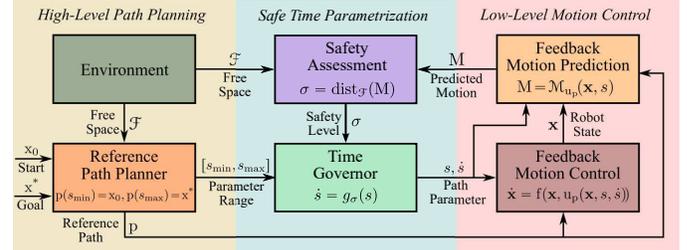} 
\caption{A time governor is a dynamical system that performs online feedback time parametrization of a reference path for safe  path-following control around obstacles based on the safety assessment of predicted robot motion.}
\label{fig.TimeGovernor}
\end{figure}

\subsection{Motivation and Related Literature}
\label{sec.Literature}

\subsubsection{Path Planning \& Following versus Trajectory Planning \& Tracking}

Motion representation plays a key role in robot motion planning and control.
As a standard motion representation, a \emph{path} is a purely geometric description of robot motion as a sequence of robot configurations (e.g., position) without a specific timing requirement, whereas a \emph{trajectory} is a time sequence of robot configurations that describes a desired robot motion \cite{choset_etal_PrinciplesOfRobotMotion2005, siciliano_etal_RoboticsModellingPlanningControl2009, lynch_park_ModernRobotics2017}.
Hence, a trajectory can be considered as a combination of a path and its time parametrization (a.k.a. time scaling \cite{dahl_nielsen_TRA1990}, time allocation \cite{richter_bry_row_ISRR2016}, and timing law \cite{aguiar_hespanha_kokotovic_Automatica2008}). 
Similar to path planning and following, trajectory planning and tracking is performed in a decoupled and sequential fashion in three steps by (i) first finding a dynamically feasible and collision-free path, (ii) and then determining its time parametrization under safety and control constraints,  (iii) and finally executing the reference trajectory using trajectory-tracking control \cite{choset_etal_PrinciplesOfRobotMotion2005, lynch_park_ModernRobotics2017}.
However, trajectory planning and tracking might be more limited than path planning and following since the resulting robot motion plan is required to be dynamically feasible and achievable at specified times which inherently causes some performance limitations for trajectory tracking that are not relevant to path following \cite{aguiar_etal_IFAC2004, aguiar_hespanha_kokotovic_Automatica2008}. 
As a coupled robot motion planning and control approach, trajectory optimization aims at constructing a dynamically feasible and safe robot trajectory (often with associated control in a model predictive control formulation \cite{neunert_etal_ICRA2016}) at once using numerical optimization tools \cite{betts_JGCD1998}, but it is usually slower for real-time application settings and its quality strongly depends on initialization \cite{choset_etal_PrinciplesOfRobotMotion2005}.
Due to their open-loop nature, trajectory planning and optimization methods often suffer from significant replanning cycles in practice for ensuring safety and high trajectory tracking performance \cite{ding_gao_wang_shen_TRO2019, tordesillas_etal_ToR2021}, which might be mitigated up to a certain level by offline numerically precomputed tracking error bounds and replanning with simplified robot models \cite{chen_etal_TAC2021}.
In this paper, we introduce a new online feedback path time-parametrization approach for provably correct and safe path-following control that allows for closing the gap between trajectory planning \& tracking and path planning \& following to leverage the strengths of both approaches.

\subsubsection{Path Time-Parametrization}

A classical way of trajectory generation is a time parametrization of a prescribed path by a monotone increasing scalar function of time that specifies at which time a robot configuration along a path is  realized \cite{choset_etal_PrinciplesOfRobotMotion2005, lynch_park_ModernRobotics2017}. 
Existing path time-parametrization methods are mainly based on either numerical optimization or analytic feedback control methods. 
Optimal time parametrization of paths aims at generating an optimal trajectory along which a given robot motion cost (e.g., travel time and/or control effort) is minimized while satisfying safety and control constraints.
Such optimal path time-parametrization methods are often applied for minimum-time robot motion design with actuation constraints, for example, for aggressive drone flight \cite{richter_bry_row_ISRR2016,  gao_wu_pan_zhou_shen_IROS2018, tordesillas_etal_ToR2021}, robotic manipulation \cite{bobrow_dubowsky_gibson_IJRR1985, shin_mckay_TAC1985},  legged locomotion \cite{pham_pham_TRO2018}, and autonomous driving \cite{lie_zhan_tomizuka_IVS2017}.  
Besides high computational cost, a limitation of these optimal path time-parametrization methods is their sensitivity to modelling errors and disturbances since an optimal trajectory often pushes safety and  system (e.g., actuation) limits; 
whereas online feedback time reparametrization of nominal time-optimal trajectories using time-scaling tracking control increases system robustness \cite{dahl_nielsen_TRA1990}.

Online feedback path time-parametrization is often performed by controlling the progress of a ``\emph{virtual target}'' along the path \cite{micaelli_samson_TechReport1993}; for example, based on the distance between the robot and the virtual target \cite{mohan_etal_Access2020, paliotta_etal_TCST2019, lefeber_ploeg_nijmeijer_IFAC2017} such that the virtual target slows down when the robot is far away from the target to avoid significant deviations from the path (a.k.a., the corner cutting problem) \cite{mohan_etal_Access2020, paliotta_etal_TCST2019}.
Feedback path time-parametrization is demonstrated to yield robust path-following control of  mobile wheeled robots \cite{aicardi_etal_RAM1995, mohan_etal_Access2020}, underactuated marine vehicles \cite{paliotta_etal_TCST2019, aicardi_etal_MCCA2001}, fixed-wind aerial vehicle \cite{flores_luco-cardenas_lozano_ICUAS2013} and vehicle platooning \cite{lefeber_ploeg_nijmeijer_IFAC2017}.
Although these feedback path time-parametrization methods allow bounding path following/tracking error (which is often used as a safety heuristic), they fail to ensure safe path-following control.
As a result, feedback path-time parametrization is often combined with obstacle avoidance  \cite{lapierre_zapata_lepinay_ICRA2007, lapierre_zapata_lepinay_IJRR2007} and model predictive control \cite{mohan_etal_Access2020} to prioritize safety while path following such that the path following task is abandoned if the safety of the robot is at risk.
In this paper, to the best of our knowledge, we show for the first time how to perform provably correct and safe path-following control with online feedback path time-parametrization using feedback motion prediction.

\subsubsection{Path-Following Control}

The objective of path-following control is to realize a given reference path as closely as possible, often using a lookahead goal selection criterion based on the reference path and robot (state) information.
A simple yet effective geometric path-following control approach is the pure pursuit strategy that determines a local goal along a reference path that is within a certain lookahead distance from the robot \cite{coulter_TechReport1992}.
The performance of the pure pursuit path-following strategy is known to significantly depend on the selection of a lookahead distance \cite{amidi_thotpe_MR1991, ohta_etal_CPSNA2016, park_lee_han_ETRIJ2015, morales_etal_JASP2009}.
While a short lookahead distance causes some stability issues with oscillatory robot motion along the reference path, a long lookahead distance often results in large deviations from the path and so the well-known corner-cutting problem.        
To reduce these stability and corner-cutting issues, several adaptive lookahead  distance selection methods are proposed based on robot velocity \cite{park_lee_han_ETRIJ2015, park_lee_han_ICCAS2014, ohta_etal_CPSNA2016, hingwe_tomizuka_ACC1998, kuwata_etal_GNC2008}, the robot's  distance to path \cite{campbell_MScThesis2007, giesbrecht_etal_TechReport2005, serna_etal_ACI2017}, heading  deviation \cite{wang_etal_CCC2019, bayuwindra_etal_TCST2019}, path curvature \cite{ahn_etal_IJAT2021},
reinforcement learning \cite{goel_chauhan_ICRS2021},  fuzzy decision rules \cite{ollero_garcia-cerezo_martinez_IFAC1993}, and  numerical optimization \cite{wang_Etal_IEEEAccess2020, sukhil_behl_arXiv2021}, as well as more advanced (e.g., vector-pursuit) path-following control approaches that use both lookahead goal position and orientation  \cite{andersen_etal_AIM2016, wit_crane_armstrong_JRS2004, park_lee_han_ETRIJ2015}.
As in virtual-agent based path time-parametrization, the safety of such geometric (pure pursuit and vector pursuit) path-following  methods is often  heuristically  addressed by minimizing path following/tracking error without any formal guarantees.
The virtual-agent interpretation of feedback path time-parametrization allow us to extend the pure pursuit approach via our time governor framework for provably correct and safe path-following control, which is also a step towards closing the gap between path-following and trajectory-tracking.%
\footnote{\label{fn.PathFollowingPathTracking}The major difference between path-following control and trajectory-tracking control is that trajectory tracking aims at asymptotically bringing a notion of a tracking error to zero with feedforward control whereas path following often has a non-zero tracking error because it purely based on negative error feedback without feedforward control.}

\subsubsection{Reference Governors}

Our design of time governors are inspired by reference governors for constrained control of dynamical systems \cite{bemporad_TAC1998, gilbert_kolmanovsky_Automatica2002, garone_nicotra_TAC2015}. 
Reference governors are applied for robot motion planning and control as a constraint-handling planning-control interface between high-level motion planning and low-level motion control to minimally modify a high-level reference plan before executing it by low-level control in order to ensure stability, safety, and  constraint satisfaction \cite{arslan_koditschek_ICRA2017, li_arslan_atanasov_ICRA2020, li_duong_atanasov_arXiv2020, isleyen_vandewouw_arslan_RAL2022, isleyen_vandewouw_arslan_arXiv2022}.  
The use of reference governors for safe robot motion planning and control suggests that the quality of feedback motion prediction (e.g., Lyapunov ellipsoids and Vandermonde simplexes) plays a key role in governed robot motion design and safety assessment, because accurate motion prediction avoids conservatism and allows for faster robot motion without compromising safety \cite{li_arslan_atanasov_ICRA2020, isleyen_vandewouw_arslan_RAL2022}.
In this paper, instead of geometrically modifying a reference plan by a reference governor, we consider the temporal adaptation of a reference path via a time governor to enable provably correct and safe path-following control. 
At a more conceptual level, we believe that time governors offers a new perspective for temporal adaptation of high-level planning for constrained control of dynamical systems.

\subsection{Contributions and Organization of the Paper}

This paper introduces a new generic time governor framework for online feedback time parametrization of a reference path for provably correct and safe path-following control around obstacles.
In \refsec{sec.TimeGovernorsProblemFormulation}, we start with a formal problem description for the design of time governors for safe path following, which shall serve a basis for further research on this topic.
In \refsec{sec.TimeGovernorsGeneralFramework}, we present a generic time governor design approach using feedback motion prediction and discuss about its safety and convergence properties.
In \refsec{sec.TimeGovernorsExample}, we provide an example construction for the fully actuated high-order  robot dynamics based on proportional-and-higher-order-derivative (PhD) path-following control and associated Lyapunov and Vandemonde motion predictions.   
In \refsec{sec.NumericalSimulations}, we systematically investigate the role of reference position and velocity feedback and motion prediction on robot behaviour and path-following performance in numerical simulations.  
We conclude in \refsec{sec.Conclusions} with a summary of our contributions and future work.

\section{Time Governors for Safe Path Following: Problem Formulation}
\label{sec.TimeGovernorsProblemFormulation}

For ease of exposition, we consider a disk-shaped robot of radius $\radius >0$ that is centered at position $\pos \in \workspace$ and moves in a known static closed bounded environment $\workspace \subseteq \R^{\dimpos}$ in the $\dimpos$-dimensional Euclidean space $\R^{\dimpos}$ that is cluttered with a collection of obstacles represented by an open set $\obstspace \subset \R^{\dimpos}$, where $\dimpos \geq 2$.
Hence, the robot's free space, denoted by  $\freespace$, of collision-free robot positions is given by
\begin{align}\label{eq.FreeSpace}
\freespace:=\clist{\pos \in \workspace \, \Big|\, \ball(\pos, \radius) \subseteq \workspace \setminus \obstspace}
\end{align}
where  $\ball(\pos,\radius):=\clist{\pos' \in \R^{\dimpos} \big| \norm{\pos' - \pos} \leq \radius }$ is the closed Euclidean ball centered at $\pos$ with radius $\radius$, and $\norm{.}$ denotes the standard Euclidean norm for both vectors and matrices.
It is also convenient to have $\freespace_{\clearance}:= \clist{\pos \in \freespace \, \big | \, \ball(\pos, \clearance) \subseteq \freespace}$ representing the $\clearance$-clearance subset of the free space $\freespace$, where $\clearance > 0$ is a fixed positive safety margin.

We assume that the robot state, denoted by $\state=(\pos, \nonpos)$, consists of  a $\dimpos$-dimensional position variable  $\pos \in \R^{\dimpos}$ and  a $\dimnonpos$-dimensional nonpositional (e.g., velocity or orientation) variable $\nonpos \in \R^{\dimnonpos}\!$, and the equation of motion can be written~as
\begin{align}\label{eq.RobotDynamics}
\dot{\state} = \begin{bmatrix} \dot{\pos} \\ \dot{\nonpos} \end{bmatrix} = \dyn(\state, \ctrl)
\end{align}
where $\dyn: \R^{\dimpos+\dimnonpos} \times \R^{\dimctrl} \rightarrow \R^{\dimpos+\dimnonpos}$ represents the state-space robot dynamics and  $\ctrl \in \R^\dimctrl$ is a $\dimctrl$-dimensional control input.
For example, for the fully actuated second-order robot model, the nonpositional state variable is the robot velocity $\nonpos = \dot{\pos}$, the control input  $\ctrl$ corresponds to the robot acceleration $\ddot{\pos} = \ctrl$,  and the robot dynamics function is $\dyn(\state, \ctrl) = \begin{bmatrix}
\dot{\pos} \\
\ctrl
\end{bmatrix}$; 
for the 2D kinematic unicycle robot model, the nonpositional state variable is the robot orientation angle $\nonpos = \orient \in \R$, the control inputs $\ctrl=(\linvel, \angvel) \in \R^2$ are the linear speed $v \in \R$ and the angular speed $\omega\in \R$, and the unicycle dynamics function is given by 
$\vect{f}(\state, \ctrl) = \scalebox{0.9}{$\begin{bmatrix}
\linvel \cos \orient \\
\linvel \sin \orient \\
\angvel
\end{bmatrix}$} $.

For path-following control, we consider a Lipschitz-conti-nuous collision-free reference path $\refpath(\pathparam):[\minpathparam, \maxpathparam] \rightarrow \freespace_{\clearance}$  in the $\clearance$-clearance\footnote{Here, we require a minimum path clearance because path-following control near collisions is difficult not only in practice but also in theory.} free space $\freespace_{\clearance}$, parametrized by a~scalar path parameter $\pathparam \in \blist{\minpathparam, \maxpathparam} \subseteq \R$, which can be either automatically generated by a standard off-the-shelf path planner \cite{choset_etal_PrinciplesOfRobotMotion2005} or manually specified by the user such that $\refpath(\pathparam) \in \freespace_{\clearance}$ for all $\pathparam \in  [\minpathparam, \maxpathparam]$.
Assuming a first-order path parameter dynamic that allows for controlling the time rate of change of path parameter, we define  the problem of path-following control with feedback path time-parametrization as follows.

\begin{definition}\label{def.PathFollowingControl}
\emph{(Simultaneous Path Time-Parametrization and Path-Following Control)} For any given Lipschitz-continuous collision-free reference path $\refpath(\pathparam)\!:\! \blist{\minpathparam, \maxpathparam} \rightarrow \freespace_{\clearance}$, \emph{integrated path-following control with feedback path time-parametrization} is the design of a Lipschitz-continuous first-order path parameter dynamic $\pathdyn(\state, \pathparam)$ and a Lipshitz-continuous path-following control policy $\pathctrl_{\refpath}(\state, \pathparam, \dot{\pathparam})$ for the robot dynamics in \refeq{eq.RobotDynamics} of the form 
\begin{subequations}\label{eq.PathParametrizationFollowingControl}
\begin{align}
\dot{\pathparam} &= \pathdyn(\state, \pathparam) \label{eq.PathParameterDynamics}\\
\dot{\state} &= \dyn(\state, \pathctrl_{\refpath}(\state, \pathparam, \dot{\pathparam})) \label{eq.PathFollowingControl}
\end{align}
\end{subequations}
such that 
\begin{enumerate}[i)]
\item \emph{(Parameter Monotonicity)} 
the path parameter dynamic $\pathdyn(\state, \pathparam)$ is nonnegative (i.e., $\pathdyn(\state, \pathparam) \geq 0$ for all $\state \in \R^{\dimpos + \dimnonpos}$ and $\pathparam \in  \blist{\minpathparam, \maxpathparam}$) and leaves $\blist{\minpathparam, \maxpathparam}$ positively invariant (i.e., $\pathdyn(\state, \maxpathparam) = 0$ for all $\state \in \R^{\dimpos + \dimnonpos}$); that is to say, the path parameter $\pathparam$ is monotone increasing and bounded over $\blist{\minpathparam, \maxpathparam}$,
\item \emph{(Point Stabilization)} the path-following control policy $\pathctrl_{\refpath}(\state,\pathparam, 0)$ is globally asymptotically stable at $\refpath(\pathparam)$ for any constant path parameter $\pathparam \in [\minpathparam, \maxpathparam] $ with  $\dot{\pathparam} = 0$; that is to say, starting at $t = 0$ from any initial robot state $\state_{0} \in \R^{\dimpos + \dimnonpos}$ and path parameter $\pathparam_{0} \in [\minpathparam, \maxpathparam]$, the robot position trajectory $\pos(t)$ of \refeq{eq.RobotDynamics}  under $\pathctrl_{\refpath}(\state,\pathparam_0, 0)$  satisfies $\lim_{t \rightarrow \infty} \pos(t) = \refpath(\pathparam_0)$.
\end{enumerate}
\end{definition}

It is important to note that if the path parameter dynamic $\pathdyn(\state, \pathparam)$ in \refeq{eq.PathParameterDynamics}  asymptotically converges to the maximum path parameter $\maxpathparam$, then this also ensures (without any safety guarantees) that the point-stabilizing path-following control policy $\pathctrl_{\refpath}(\state, \pathparam, \dot{\pathparam})$ asymptotically brings the robot position $\pos(t)$ to the end of the reference path $\refpath(\pathparam)$, i.e., $\lim_{t \rightarrow \infty} \pathparam(t) = \maxpathparam$ implies $\lim_{t \rightarrow \infty} \pos(t) = \refpath(\maxpathparam)$.  
In \refdef{def.PathFollowingControl}, a successful completion of a reference path (i.e., asymptomatically reaching to the end of a reference path) is not required,  because path-following control with prioritized  obstacle avoidance for safety often fails to satisfy such a requirement \cite{lapierre_zapata_lepinay_IJRR2007, mohan_etal_Access2020}.  
To successfully complete a path-following task, as an alternative to an open-loop constant speed profile $\pathdyn(\state, \pathparam) = \pathparamrate$  at a fixed desired path parameter speed $\pathparamrate > 0$ \cite{koh_cho_JIRS1999}, a common choice of feedback path time-parametrization uses a smoothly saturated path-following error (e.g., the distance between  the robot position $\pos$ and the reference path point $\refpath(\pathparam)$) via a sigmoid (e.g. tangent hyperbolic $\tanh$) function as\reffn{fn.StableMonotoneTimeParametrization} \cite{paliotta_etal_TCST2019, mohan_etal_Access2020}
\begin{align}\label{eq.ExamplePathParameterDynamics}
\pathdyn(\state, \pathparam) = \pathparamrate \plist{1 - \eta_{\pathparam} \tanh\norm{\pos - \refpath(\pathparam)}}
\end{align}
so that the progress of path parameter is adaptively adjusted in order to keep the path-following error bounded, where  $\eta_{\pathparam} \in [0,1]$ is a constant saturation magnitude.
In addition to achieving a certain level of path-following performance, such a feedback path time-parametrization approach in \refeq{eq.ExamplePathParameterDynamics} often use  the bounded path-following error as a safety heuristic (without any formal guarantees) based on the  assumption that the reference path has enough clearance from obstacles. 

\addtocounter{footnote}{1}
\footnotetext{\label{fn.StableMonotoneTimeParametrization}Any nonnegative path parameter dynamic $\pathdyn(\state, \pathparam)$, for example, \refeq{eq.ExamplePathParameterDynamics}, can be modified while preserving Lipschitz continuity and nonnegativity as 
\begin{align}
\hat{\pathdyn}(\state, \pathparam) = \min \plist{\pathdyn(\state, \pathparam), -\gain_\pathparam(\pathparam - \maxpathparam)} \nonumber
\end{align}
so that $\hat{\pathdyn}(\state, \pathparam)$ leaves $[\minpathparam, \maxpathparam]$ positively invariant and is globally stable at $\pathparam = \maxpathparam$, where $\gain_\pathparam >0$ is a constant control gain. 
}

Ideally, path-following control of a robot around obstacles should ensure robot safety at all times, as well as a successful completion of a reference path, if possible.
Accordingly, given a point-stabilizing path-following control (\refdef{def.PathFollowingControl}), in this paper, we consider the design of a time governor for provably correct and safe path-following control, as described below. 
\begin{definition}\label{def.TimeGovernor} 
\emph{(Time Governors for Safe Path-Following Control)}
Given a Lipschitz-continuous $\clearance$-clearance reference path $\refpath(\pathparam):[\minpathparam, \maxpathparam] \rightarrow \freespace_{\clearance}$ in the free space $\freespace$ and an associated Lipschitz-continuous point-stabilizing path-following control $\pathctrl_{\refpath}(\state, \pathparam, \dot{\pathparam})$ for the robot dynamics in \refeq{eq.RobotDynamics}, a \emph{time governor for safe path-following control} is the design of a first-order nonnegative path parameter dynamic 
\begin{align}
\dot{\pathparam} &= \pathdyn(\state, \pathparam) \geq 0
\end{align}
such that, starting at time $t=0$ from $\pathparam(0) = \minpathparam$ and some state $\state(0) \in \R^{\dimpos + \dimnonpos}$ with\footnote{Here, any initial robot state that results in safe robot motion towards the initial reference path point can be considered as in Propositions \ref{prop.Safety} \& \ref{prop.Convergence}.} $\pos(0) = \refpath(\minpathparam)$,  the path parameter trajectory $\pathparam(t)$ and the robot position trajectory $\pos(t)$ of the closed-loop robot dynamics in \refeq{eq.PathFollowingControl} under path following control $\pathctrl_{\refpath}(\state, \pathparam, \dot{\pathparam})$ asymptotically converge to the end of the reference path $\refpath(\pathparam)$ with no collision between the robot and obstacles along the way, i.e.,
\begin{subequations}
\begin{align}
\pos(t) &\in \freespace \quad \forall t \geq 0, \\  
\lim_{t \rightarrow \infty} \pathparam(t) &= \maxpathparam, \\
\lim_{t \rightarrow \infty} \pos(t) &= \refpath(\maxpathparam).
\end{align}
\end{subequations}
\end{definition}

It is useful to observe  that \refdef{def.TimeGovernor} allows one to systematically separate and handle stability, safety and task constraints at different levels for robot motion planning and control:  the stability of robot motion is ensured by low-level path-following control, the safety of robot motion is handled by a time governor, and the task requirements are resolved by a high-level reference path~planner, as illustrated in \mbox{\reffig{fig.TimeGovernor}}.

\section{Time Governors for Safe Path Following: General Framework}
\label{sec.TimeGovernorsGeneralFramework}

In this section we describe the generic building blocks of our time governor framework for provably correct and safe path-following control, illustrated in \reffig{fig.TimeGovernor}, and present its important safety and convergence properties.

\subsection{Fundamental Elements for Safe Path-Following Control} 

Our time governor framework for safe path planning and path following control consists of five building elements: reference path planner, path-following control, feedback motion prediction, motion safety assessment, and a time governor, whose specific roles and requirements are presented below. 

\subsubsection{Reference Path Planner}

For any given pair of collision-free start and goal positions $\pos_{\mathrm{start}}, \pos_{\mathrm{goal}} \in \freespace_{\clearance}$, if exists, a \emph{reference path planner} generates a Lipschitz-continuous collision-free reference path \mbox{$\refpath(\pathparam): [\minpathparam, \maxpathparam] \rightarrow \freespace_{\clearance}$} in the robot's free space $\freespace$ such that   $\refpath(\minpathparam) = \pos_{\mathrm{start}}$, $\refpath(\maxpathparam) = \pos_{\mathrm{goal}}$, and the reference path $\refpath(\pathparam)$ has an $\epsilon$-clearance from the free space boundary $\partial \freespace$, i.e., $\dist_{\freespace}(\refpath([\minpathparam, \maxpathparam])) \geq  \epsilon $, where $\epsilon>0$ is a constant positive safety margin, and the collision distance $\dist_{\freespace}(\mathcal{A})$ of  set $\mathcal{A} \subseteq \R^{\dimpos}$ to  the free space boundary $\partial \freespace$ is defined as
\begin{align} \label{eq.Distance2Collision}
\dist_{\freespace}(\mathcal{A}) := \left \{ 
\begin{array}{@{}c@{}l@{}}
\min\limits_{\substack{\vect{a} \in \mathcal{A} \\ \vect{b} \in \partial \freespace}} \norm{\vect{a} - \vect{b}} & \text{, if } \mathcal{A} \subseteq \freespace, \\
0 & \text{, otherwise.}
\end{array}
\right.
\end{align}    
Here, a collision distance of zero means unsafe; and the higher the distance-to-collision is the safer. 
Note that we consider being exactly on the boundary of the free space to be unsafe although it is, by definition in \refeq{eq.FreeSpace}, free of collisions.

\subsubsection{Path-Following Control}
\label{sec.PathFollowingControl}

A path-following control policy $\pathctrl_{\refpath}(\state, \pathparam, \dot{\pathparam}): \R^{\dimpos + \dimnonpos} \times [\minpathparam, \maxpathparam] \times \R \rightarrow \R^{\dimctrl}$, associated with a reference path $\refpath(\pathparam):[\minpathparam, \maxpathparam] \rightarrow \R^{\dimpos}$,  is a Lipschitz-continuous point-stabilizing  control law  for the robot dynamics in \refeq{eq.RobotDynamics} (see \refdef{def.PathFollowingControl})  such that the closed-loop robot motion under path-following control $\pathctrl_{\refpath}(\state, \pathparam, 0)$ is globally asymptomatically stable at position $\refpath(\pathparam)$  for any  path parameter $\pathparam \in [\minpathparam, \maxpathparam]$ and $\dot{\pathparam} = 0$.
For example, let $\navctrl_{\goal}(\state)\!:\!\R^{\dimpos + \dimnonpos} \!\rightarrow\! \R^{\dimctrl}$ be a globally stable navigation controller for the robot model in \refeq{eq.RobotDynamics} that asymptomatically brings all robot states $\state \!\in\! \R^{\dimpos + \dimnonpos}$ to a given goal position $\goal \!\in \!\R^{\dimpos}$, then one can construct a point-stabilizing path-following controller $\pathctrl_{\refpath}(\state, \pathparam, \dot{\pathparam})$ using the globally stable navigation controller $\navctrl_{\goal}(\state)$ as $\pathctrl_{\refpath}(\state, \pathparam, \dot{\pathparam}) := \navctrl_{\refpath(\pathparam)}(\state)$.
Moreover, one can also use the path parameter rate $\dot{\pathparam}$ in path-following control design to improve path-following performance, as in path-tracking control\reffn{fn.PathFollowingPathTracking}. 
In general, as a choice of point-stabilizing path-following control, one may consider any preferred feedback motion controller for the robot dynamics of interest as long as the resulting feedback robot motion can be accurately predicted in terms of the robot state $\state$ and the reference path point~$\refpath(\pathparam)$, as described below.

\subsubsection{Feedback Motion Prediction of Path-Following Control}
\label{sec.FeedbackMotionPrediction}

Feedback motion prediction of an autonomous robotic system operating under a specific control policy is finding a motion set (e.g., a positively invariant Lyapunov level set) that contains the closed-loop robot motion trajectory starting from a given initial robot state  \cite{isleyen_vandewouw_arslan_RAL2022, isleyen_vandewouw_arslan_arXiv2022, arslan_isleyen_2022}.
For path-following control with feedback path-time parametrization, constructing a feedback motion prediction algorithm that accurately describes the entire path-following motion of a robotic system is difficult, because the robot's path-following motion strongly depends on the reference path profile and its time parametrization. 
Hence, we below consider the local feedback motion prediction of path-following control using the point stabilization property.

By assuming that the path parameter $\pathparam \in [\minpathparam, \maxpathparam]$ would be kept constant, i.e. $\dot{\pathparam} = 0$, for instance, to ensure robot safety, a feedback motion prediction set, denoted by $\motionset_{\pathctrl_{\refpath}}(\state, \pathparam) $, of a Lipschitz-continuous point-stabilizing  path-following control policy $\pathctrl_{\refpath}(\state, \pathparam, \dot{\pathparam})$, associated with a Lipschitz-continuous reference path $\refpath(\pathparam):[\minpathparam, \maxpathparam] \rightarrow \R^{\dimpos}$, is defined to be a closed set that contains the closed-loop position trajectory $\pos(t)$ of the robot model in \refeq{eq.RobotDynamics} under the control law $\pathctrl_{\refpath}(\state, \pathparam, 0)$ for all future times, i.e., starting at $t=0$ from an initial state $\state(0) \in \R^{\dimpos+\dimnonpos}$ 
\begin{align}
\pos(t) \in \motionset_{\pathctrl_{\refpath}}(\state(0), \pathparam) \quad \forall t \geq 0
\end{align} 
such that its radius $\setradius_{\refpath(\pathparam)}\!\plist{\motionset_{\pathctrl_{\refpath}}\!(\state, \pathparam)\!}$ relative to the reference path point $\refpath(\pathparam)$ is Lipschitz continuous with respect to robot state $\state$ and path parameter $\pathparam$, and  asymptotically converges to zero along the closed-loop robot state trajectory $\state(t)$, i.e., 
\begin{align}\label{eq.DecayingMotionPredictionRadius}
\lim_{t \rightarrow \infty} \setradius_{\refpath(\pathparam)}\plist{\motionset_{\pathctrl_{\refpath}}(\state(t), \pathparam)} = 0
\end{align}
where  the radius of set $\mathcal{A} \subseteq \R^{\dimpos}$ relative to point $\vect{b} \in \R^{\dimpos}$  is defined as 
\begin{align}\label{eq.SetRadius}
\setradius_{\vect{b}}(\mathcal{A}) := \max_{\vect{a} \in \mathcal{A}} \norm{\vect{a} - \vect{b}}.
\end{align}
Note that the Lipschitz continuity of motion prediction radius $\setradius_{\refpath(\pathparam)}\!\plist{\motionset_{\pathctrl_{\refpath}}\!(\state, \pathparam)\!}$ with respect to the robot state $\state$ implies that the motion prediction set $\motionset_{\pathctrl_{\refpath}}(\state, \pathparam)$ is radially bounded relative to a stable robot state $(\refpath(\pathparam), \nonpos)$ under the control law $\pathctrl_{\refpath}(\state, \pathparam,0)$ as\footnote{It follows from the the Lipschitz continuity  of $\setradius_{\refpath(\pathparam)}\!\plist{\motionset_{\pathctrl_{\refpath}}\!(\state, \pathparam)\!}$ that $\absval{\setradius_{\refpath(\pathparam)}\!\plist{\motionset_{\pathctrl_{\refpath}}\!(\state, \pathparam)\!} - \setradius_{\refpath(\pathparam)}\!\plist{\motionset_{\pathctrl_{\refpath}}\!((\refpath(\pathparam), \nonpos), \pathparam)\!}} \leq \eta \norm{\state - \!(\refpath(\pathparam), \nonpos)}$ where $\setradius_{\refpath(\pathparam)}\!\plist{\motionset_{\pathctrl_{\refpath}}\!((\refpath(\pathparam), \nonpos), \pathparam)\!} =0$ because  the motion prediction radius is asymptotically decaying to zero and $(\refpath(\pathparam), \nonpos)$ is stable under $\pathctrl_{\refpath}(\state, \pathparam, 0)$. 
Hence, $\setradius_{\refpath(\pathparam)}\!\plist{\motionset_{\pathctrl_{\refpath}}(\state, \pathparam)\!} \leq \eta \norm{\state - (\refpath(\pathparam), \nonpos)}$ is equivalent to saying that $\motionset_{\pathctrl_{\refpath}}\!(\state, \pathparam)\! \subseteq\! \ball(\refpath(\pathparam), \eta \norm{\state - (\refpath(\pathparam), \nonpos)})$.
Overall, the Lipschitz-continuity of motion prediction radius is an elegant way of stating the radial boundedness of  motion prediction without referring to a stable robot state \cite{isleyen_vandewouw_arslan_RAL2022}.
} 
\mbox{$\motionset_{\pathctrl_{\refpath}}(\state, \pathparam) \subseteq \ball(\refpath(\pathparam), \eta \norm{\state - (\refpath(\pathparam), \nonpos)})$}, where $\eta$ is the associated Lipschitz constant. 
Moreover, it also follows from the asymptotic decay property of the motion prediction radius in \refeq{eq.DecayingMotionPredictionRadius} that the motion prediction set asymptotically converges to the single point at the reference point $\refpath(\pathparam)$, i.e., $\lim_{t \rightarrow \infty} \motionset_{\pathctrl_{\refpath}}(\state(t), \pathparam) = \clist{\refpath(\pathparam)}$, because the feedback motion control $\pathctrl_{\refpath}(\state, \pathparam,0)$ is globally asymptotically stable at position $\refpath(\pathparam)$ (see \refdef{def.PathFollowingControl}).

Hence, feedback motion prediction enable us to predict the path-following performance of a robot.
\begin{corollary}
Since a feedback motion prediction set $\motionset_{\pathctrl_{\refpath}}(\state, \pathparam)$ contains both the robot position $\pos$ and the reference path point $\refpath(\pathparam)$, the distance of the robot to the reference path is bounded by the radius of motion prediction set, i.e.,
\begin{align}
\norm{\pos - \refpath(\pathparam)} \leq \setradius_{\refpath(\pathparam)}\plist{\motionset_{\pathctrl_{\refpath}}(\state, \pathparam)}.
\end{align}
Moreover, if the path parameter is kept constant, i.e., $\dot{\pathparam} =0$, the robot position trajectory $\pos(t)$ under the path-following control $\pathctrl_{\refpath}(\state, \pathparam, 0)$ has a bounded path-following error in terms of the motion prediction radius as
\begin{align}
\norm{\pos(t) - \refpath(\pathparam)} \leq \setradius_{\refpath(\pathparam)}\plist{\motionset_{\pathctrl_{\refpath}}(\state(0), \pathparam)} \quad \forall t \geq 0.
\end{align}
\end{corollary}

More importantly, feedback motion prediction allows  one to effectively identify safe robot states of complex robotic systems (e.g., a challenge of kinodynamic planning of highly dynamic robots \cite{lavalle_kuffner_IJRR2001}), because having the motion prediction set in the free space implies safe robot motion \cite{isleyen_vandewouw_arslan_RAL2022, isleyen_vandewouw_arslan_arXiv2022}.
\begin{corollary}
When the path parameter is kept constant, i.e., $\dot{\pathparam} =0$, a robot state $\state \in \R^{\dimpos + \dimnonpos}$ of the robot model in \refeq{eq.RobotDynamics} under path-following control $\pathctrl_{\refpath}(\state, \pathparam, 0)$ is safe if there exists a path parameter $\pathparam \in [\minpathparam, \maxpathparam]$ such that the associated motion prediction $\motionset_{\pathctrl_{\refpath}}(\state, \pathparam)$ is in the free space $\freespace$, because the robot position trajectory $\pos(t)$ starting at $t = 0$ from the robot state $\state$  satisfies $ \pos(t) \in \motionset_{\pathctrl_{\refpath}}(\state, \pathparam) \subseteq \freespace$ for all $ t \geq 0$.
\end{corollary}

Accordingly, we below continuously monitor the safety of the closed-loop robot motion under path-following control by measuring the distance of the motion prediction set to the free space boundary, and adjust (if needed, stop) the progress of path parameter to ensure safe path-following control.

\subsubsection{Safety Assessment via Feedback Motion Prediction}

Given a feedback motion prediction set $\motionset_{\pathctrl_{\refpath}}(\state, \pathparam)$ for the robot model in \refeq{eq.RobotDynamics} under path-following control $\pathctrl_{\refpath}(\state, \pathparam, \dot{\pathparam})$, the safety level $\safetylevel(\state, \pathparam)$ of the robot's path-following motion starting from state $\state$ towards the reference path point $\refpath(\pathparam)$ is measured by the minimum distance of the predicted robot motion set $\motionset_{\pathctrl_{\refpath}}(\state, \pathparam)$ to the free space boundary $\partial \freespace$ as
\begin{align}
\safetylevel(\state,\pathparam):= \dist_{\freespace}\plist{\motionset_{\pathctrl_{\refpath}}(\state, \pathparam)}
\end{align}   
where the collision distance $\dist_{\freespace}$ is defined as in \refeq{eq.Distance2Collision}.

A technical requirement of the safety level measure (and so feedback motion prediction) for stability analysis is that $\safetylevel(\state, \pathparam)$ is a locally Lipschitz continuous function of the robot state $\state$ and the path parameter $\pathparam$. 
For example, for any Lipschitz-continuous reference path $\refpath(\pathparam)$, the safety level $\safetylevel(\state, \pathparam)$ is locally Lipschitz continuous if the motion prediction set $\motionset_{\pathctrl_{\refpath}}(\state, \pathparam)$ can be expressed as an affine transformation of some fixed sets (e.g., the unit ball/simplex) that is a smooth function of the robot state $\state$ and the reference path point $\refpath(\pathparam)$ (see Lemma 1 in \cite{isleyen_vandewouw_arslan_RAL2022}).

\subsubsection{Time Governor for Safe Path-Following Control}
%
In accordance with \refdef{def.TimeGovernor}, using the safety level $\safetylevel(\state, \pathparam)$ of the predicted robot motion set $\motionset_{\pathctrl_{\refpath}}\!(\state, \pathparam)$ of path-following control $\pathctrl_{\refpath}(\state, \pathparam, \dot{\pathparam})$, we design a nonnegative time governor function $\pathdyn_{\safetylevel}(\state, \pathparam)$ that determines the time rate of change of the path parameter $\pathparam$ over $[\minpathparam, \maxpathparam]$ for any robot state $\state \in \R^{\dimpos + \dimnonpos}$  as\reffn{fn.StableMonotoneTimeParametrization}
\begin{align}\label{eq.SafeTimeGovernor}
\pathdyn_{\safetylevel}(\state, \pathparam) = \min\plist{ \gain_{\safetylevel}\dist_{\freespace}\plist{\motionset_{\pathctrl_{\refpath}}\!(\state, \pathparam)\!}, - \gain_{\pathparam}(\pathparam - \maxpathparam)\!} \!\!
\end{align}  
where $\gain_{\safetylevel}, \gain_{\pathparam} > 0$ are fixed positive control coefficients.
By construction, the time governor $\pathdyn_{\safetylevel}(\state, \pathparam)$ in \refeq{eq.SafeTimeGovernor} ensures that the path parameter $\pathparam$ is increasing if and only if the predicted  robot motion is safe, i.e., $\safetylevel(\state, \pathparam)= \dist\plist{\motionset_{\pathctrl_{\refpath}}(\state, \pathparam), \partial \freespace} >0 $. 
Also note that the time governor  $\pathdyn_{\safetylevel}(\state, \pathparam)$ is Lipschitz continuous since the safety level $\safetylevel(\state, \pathparam)$ is assumed to be Lipschitz and the minimum of Lipschitz continuous functions is again Lipschitz continuous  \cite{hager_JCO1979}.%
\footnote{\label{fn.MinimumLipschitzContinuity}The minimum of Lipschitz continuous functions are Lipschitz because a pair of scalar-valued functions $f$ and $g$  satisfy $\min(f,g) = \frac{f + g - \absval{f- g}}{2} $.}

\subsection{Safety and Convergence of Path-Following Control}

To prove the correctness of our time governor framework for safe path-following control, we below show  that the time governor design in \refeq{eq.SafeTimeGovernor} ensures that the robot motion under time-governed path-following control is always safe, and  both the path parameter and the robot position asymptotically reach to the end of a collision-free reference path.

\begin{proposition}\label{prop.Safety} \emph{(Safety)}
Starting at $t=0$ from any robot state $\state(0)\in \R^{\dimpos + \dimnonpos}$ and  path parameter 
\mbox{$\pathparam(0) \in  [\minpathparam, \maxpathparam]$} with  $\motionset_{\pathctrl_{\refpath}}\plist{\state(0), \pathparam(0)} \subseteq \freespace $, the position trajectory $\pos(t)$ of the robot model in \refeq{eq.RobotDynamics} under a Lipschitz-continuous point-stabilizing  path-following control policy $\pathctrl_{\refpath}(\state, \pathparam, \dot{\pathparam})$, associated with a Lipschitz-continuous reference path \mbox{$\refpath(\pathparam)\!:\![\minpathparam, \maxpathparam] \rightarrow \R^{\dimpos}$} and the Lipschitz-continuous time governor $\dot{\pathparam}\!=\! \pathdyn_{\safetylevel}(\state, \pathparam)$ in \refeq{eq.SafeTimeGovernor}, stays collision-free in the free space  $\freespace$ for all future~times, i.e.,
\begin{align}
\motionset_{\pathctrl_{\refpath}}\plist{\state(0), \pathparam(0)} \subseteq \freespace  \, \Longrightarrow \,  \pos(t) \in \freespace \quad \forall t \geq 0.
\end{align}    
\end{proposition}
\begin{proof}
Let $\state(t)$ and $\pathparam(t)$, respectively, denote the robot state trajectory and the path parameter trajectory under the path-following control $\pathctrl_{\refpath}(\state, \pathparam, \dot{\pathparam})$ and the time governor $\pathdyn_{\safetylevel}(\state, \pathparam)$.
Since the collision distance $\dist_{\freespace}\plist{\motionset_{\pathctrl_{\refpath}}(\state, \pathparam)}$ of the predicted robot motion set is Lipschitz continuous, one can consider a partition of the time interval $[0, \infty)$ based on distinct time instances $\plist{t_0 = 0, t_1, \ldots, t_i, \ldots }$  with $t_i < t_{i+1}$ such that $\dist_{\freespace}\plist{\motionset_{\pathctrl_{\refpath}}(\state(t), \pathparam(t))}$ is either positive or zero over $[t_i, t_{i+1})$ and changes its sign between consecutive intervals.
Given such a time interval $[t_i, t_{i+1})$:
\begin{itemize}
\item  If  $\dist_{\freespace}\plist{\motionset_{\pathctrl_{\refpath}}(\state(t), \pathparam(t))} > 0$ for any $t \in [t_i, t_{i+1})$, then  both the motion prediction set and the robot position  are in the free space interior $\mathring{\freespace}$, i.e., \mbox{$\motionset_{\pathctrl_{\refpath}}(\state(t), \pathparam(t)) \subseteq \mathring{\freespace}$} and $\pos(t) \in \mathring{\freespace}$, because, by definition, $\pos(t)  \in \motionset_{\pathctrl_{\refpath}}(\state(t), \pathparam(t))$.
\item Otherwise, for the case of $\dist_{\freespace}\plist{\motionset_{\pathctrl_{\refpath}}(\state(t), \pathparam(t))}= 0$ for any $t \!\in\! [t_i, t_{i+1})$, we have \mbox{$\motionset_{\pathctrl_{\refpath}}(\state(t_i), \pathparam(t_i)) \subseteq \freespace$}, because for $t_i = 0$, the initial condition satisfies $\motionset_{\pathctrl_{\refpath}}(\state(0), \pathparam(0)) \subseteq \freespace$, and for $t_i >0$, the Lipschitz-continuous collision distance over the previous time interval $[t_{i-1}, t_i)$ is strictly positive, i.e., $\dist_{\freespace}\plist{\motionset_{\pathctrl_{\refpath}}(\state(t), \pathparam(t))} > 0$ for any $t \in [t_{i-1}, t_i)$.
Hence, it follows from  \mbox{$\dot{\pathparam}(t) \! =\! \pathdyn_{\safetylevel}(\state(t), \pathparam(t))\! =\! 0$}  over $[t_i, t_{i+1})$ that under the path-following control policy $\pathctrl_{\refpath}(\state, \pathparam(t_i), 0)$, the robot position trajectory satisfies  $\pos(t) \in \motionset_{\pathctrl_{\refpath}}(\state(t_i), \pathparam(t_i)) \subseteq \freespace$ for all $t \in [t_i, t_{i+1})$.
\end{itemize}
Thus, the robot motion is free of collisions  under the time-governed path-following control, which completes the proof. \qedhere
\end{proof}

\begin{proposition}\label{prop.Convergence} (\emph{Convergence})
Starting at $t=0$ from any robot state $\state(0) \in \R^{\dimpos + \dimnonpos}$ and  path parameter $\pathparam(0) \in [\minpathparam, \maxpathparam]$ with  $\motionset_{\pathctrl_{\refpath}}\plist{\state_0, \pathparam_0} \subseteq \freespace $, under a Lipschitz-continuous point-stabilizing path-following control policy $\pathctrl_{\refpath}(\state, \pathparam, \dot{\pathparam})$ associated with a Lipschitz-continuous  $\clearance$-clearance  reference path $\refpath(\pathparam):[\minpathparam, \maxpathparam] \rightarrow \freespace_{\clearance}$ and the time governor $\dot{\pathparam}= \pathdyn_{\safetylevel}(\state, \pathparam)$ in \refeq{eq.SafeTimeGovernor}, 
both the path parameter $\pathparam(t)$ and the robot position $\pos(t)$ asymptotically converge to the end of the reference path $\refpath(\pathparam)$, i.e.,
\begin{subequations}\label{eq.Convergence}
\begin{align}
\lim_{t \rightarrow \infty} \pathparam(t) &= \maxpathparam, 
\\
\lim_{t \rightarrow \infty} \pos(t) &= \refpath(\maxpathparam).
\end{align} 
\end{subequations}
\end{proposition}
\begin{proof}
Consider $V(\state, \pathparam) = \tfrac{1}{2}(\pathparam - \maxpathparam)^2$ as a Lyapunov function whose time rate of change satisfies $\dot{V}(\state, \pathparam) = (\pathparam - \maxpathparam) \pathdyn_{\safetylevel}(\state, \pathparam) \leq 0$ for any $\state \in \R^{\dimpos + \dimnonpos}$ and $\pathparam \in [\minpathparam, \maxpathparam]$.
Note that $\dot{V}(\state, \pathparam) = 0$ if and only if $\pathparam = \maxpathparam$ or $\dist_{\freespace}(\motionset_{\pathctrl_{\refpath}}(\state, \pathparam)) = 0$. 
By construction \refeq{eq.SafeTimeGovernor}, $\maxpathparam$ is stable under the time governor $\pathdyn_{\safetylevel}(\state, \pathparam)$.
In order to conclude that $\maxpathparam$ is the only stable path parameter under the time governor $\pathdyn_{\safetylevel}(\state, \pathparam)$, observe that $\dist_{\freespace}(\motionset_{\pathctrl_{\refpath}}(\state(t), \pathparam(t)))$ might be zero only for a finite time 
since the predicted motion set radius $\setradius_{\refpath(\pathparam)}\plist{\motionset_{\pathctrl_{\refpath}}(\state, \pathparam)}$ is Lipschitz continuous and asymptotically decaying to zero in \refeq{eq.DecayingMotionPredictionRadius} under $\pathctrl_{\refpath}(\state, \pathparam, 0)$, and the reference path has a positive $\clearance$-clearance from the free space boundary, i.e. $\dist_{\freespace}(\refpath(\pathparam))\geq \clearance > 0$. 
Hence, the time governor $\pathdyn_{\safetylevel}(\state, \pathparam)$ might be zero only for a finite time duration  away from $\maxpathparam$, and always stays strictly positive for at least some finite time due to its Lipschitz continuity.  
Thus, since the path-following control $\pathctrl_{\refpath}(\state, \pathparam, \dot{\pathparam})$ is point stabilizing, it follows from LaSalle's invariance principle \cite{khalil_NonlinearSystems2001} that both the path parameter and the robot position asymptomatically converges to the end of the reference path as specified in \refeq{eq.Convergence}, which completes the proof. \qedhere 
\end{proof}

\section{Time Governors for Safe Path Following:\\PhD Motion Control and Prediction Examples}
\label{sec.TimeGovernorsExample}

In this section, we present example feedback motion control and prediction methods to demonstrate an application of our time governor framework for safe path-following control of the fully actuated high-order robot dynamics via proportional high-order  derivative (PhD) feedback control.    

\subsection{PhD Path-Following Control of  Highly Dynamic Robots}

A standard control design approach for complex nonlinear robotic systems is to simplify the actual robot dynamics to  the fully actuated $\order^{\text{th}}$-order robot model, where $\order \! \geq\! 1$, via feedback linearization \cite{deluca_oriolo_vendittelli_SYROCO2000} or differential flatness \cite{zhou_schwager_ICRA2014}.
Accordingly, based on a classical negative error feedback approach \cite{khalil_NonlinearSystems2001}, we consider the fully actuated $\order^{\text{th}}$-order robot model  that follows a given (piecewise) smooth reference path $\refpath(s)\!:\![\minpathparam, \maxpathparam] \rightarrow \R^{\dimpos}$ under a point-stabilizing proportional-higher-order-derivative (PhD) path-following control $\phdctrl_{\refpath}(\state, \pathparam, \dot{\pathparam})$
as
\begin{align}\label{eq.PhDPathFollowingControl}
\pos^{(\order)} &= \phdctrl_{\refpath}(\state, \pathparam, \dot{\pathparam})
 := \! - \!\! \sum_{k = 1}^{\order - 1} \!\phdgain_{k,\phdroots}\pos^{(k)} \! + \phdgain_{0, \phdroots} \refpath(s) \!+ \phdgain_{1, \phdroots} \tfrac{\partial \refpath(\pathparam)}{\partial \pathparam} \dot{\pathparam}
\end{align}
where $\pos^{(k)}=\frac{\diff^k}{\diff t^k} \pos$ denotes the $k^{\text{th}}$ time derivative of the robot position $\pos \! \in\! \R^{\dimpos}$, $\state=\plist{\pos^{(0)}, \pos^{(1)}, \ldots, \pos^{(\order-1)} } \!\in\! \R^{\order \dimpos}$ is the robot state%
\footnote{Note that for the fully-actuated $\order^{\text{th}}$-order robot model,  the nonpositional robot variable $\nonpos$ referred in \refeq{eq.RobotDynamics} corresponds to the robot velocity, acceleration and other high-order position derivatives, i.e., $\nonpos = \plist{\pos^{(1)}, \ldots, \pos^{(\order-1)}}$.} 
consisting of its position, velocity and so on, and $\phdgain_{0, \phdroots}, \ldots, \phdgain_{\order-1, \phdroots} \in \R$ are fixed scalar control gains that results in  negative real characteristic polynomial roots%
\footnote{Negative real characteristic roots ensure nonovershooting PhD control for the $\order^{\text{th}}$-order robot dynamics  to avoid oscillatory motion, which is a generalized notion of underdamped PD control of the second-order systems~\cite{isleyen_vandewouw_arslan_RAL2022}.} 
$\phdroots=\plist{\phdroot_1, \ldots, \phdroot_{\order}} \in \R_{<0}^{\order}$ for the closed-loop robot dynamics in \refeq{eq.PhDPathFollowingControl} and so they uniquely satisfy 
\begin{align}\label{eq.PhDRootsGains}
\prod_{k=1}^{\order} (\phdroot - \phdroot_k) = \sum_{k=0}^{\order} \phdgain_{k, \phdroots}\phdroot^k.
\end{align}
with $\phdgain_{\order, \phdroots} = 1$.
Note that the negative characteristic polynomial roots ensure the point stabilizing property of PhD path-following control  in \refeq{eq.PhDPathFollowingControl} such that when the path parameter is keep constant, i.e., $\dot{\pathparam} = 0$, the closed-loop robot dynamics under $\phdctrl_{\refpath}(\state, \pathparam, 0)$ in \refeq{eq.PhDPathFollowingControl} is globally asymptotically stable at position $\pos = \refpath(\pathparam)$ with zero velocity and high-order position derivatives, i.e., $\pos^{(k)} = 0$ for all $k = 1, \ldots, \order -1$.
Moreover, one can also represent the PhD path-following dynamics in \refeq{eq.PhDPathFollowingControl} as a first-order higher-dimensional linear dynamical system in the state space as
\begin{align}\label{eq.PhDPathFollowingStateSpace}
\dot{\state} = (\phdmat_{\phdroots} \otimes \mat{I}_{\dimpos \times \dimpos}) \plist{\state - \tr{\blist{\refpath(\pathparam), \tfrac{\diff \refpath(s)}{\diff \pathparam} \dot{\pathparam}, 0, \ldots, 0}}}
\end{align}
where $\otimes$ denotes the Kronecker product and $\phdmat_{\phdroots} \in \R^{\order \times \order}$ is the companion matrix with eigenvalues $\phdroots$ that is defined as 
\begin{align}
\phdmat_{\phdroots}:= \begin{bmatrix}
0 & 1 & 0 & \hdots & 0 \\
0 & 0 & 1 & \hdots & 0 \\
\vdots & \vdots & \vdots & \ddots & \vdots \\
0 & 0 & 0  & \hdots & 1 \\
- \phdgain_{0, \phdroots} & - \phdgain_{1, \phdroots} & - \phdgain_{2, \phdroots}, & \hdots & -\phdgain_{\order-1, \phdroots}  
\end{bmatrix}.
\end{align}

\subsection{Convex Motion Predictions for PhD Path Following}

An attractive feature of PhD path-following control in \refeq{eq.PhDPathFollowingControl} is that its feedback motion prediction can be analytically performed using the classical Lyapunov ellipsoids and the recently proposed Vandermonde simplexes \cite{isleyen_vandewouw_arslan_RAL2022}.

\subsubsection{Lyapunov Motion Prediction for PhD Path Following}

Due to the Lyapunov stability of  the linear PhD path-following dynamics in \refeq{eq.PhDPathFollowingStateSpace}, starting from any initial robot state \mbox{$\state_{0} \in \R^{\order \dimpos}$}, the robot position trajectory $\pos(t)$ under the PhD path-following control $\phdctrl_{\refpath}(\state, \pathparam, 0)$  can be bounded for all future times by the projected Lyapunov ellipsoid $\motionelp_{\phdctrl_{\refpath}}(\state_{0}, \pathparam)$ that is defined as \cite{isleyen_vandewouw_arslan_RAL2022}
\begin{align}
\motionelp_{\phdctrl_{\refpath}}(\state, \pathparam):= \elp\plist{\refpath(\pathparam), \tr{\mat{I}}_{\order \dimpos \times \dimpos} \lyapmat^{-1} \mat{I}_{\order \dimpos \times \dimpos}, \norm{\state \! -\! \boldsymbol{\refpath}(\pathparam)}_{\lyapmat}
} \nonumber
\end{align}
where $\elp(\elpctr, \elpmat, \elprad) := \clist{\elpctr + \elprad \elpmat^{\frac{1}{2}} \vect{x} \big| \vect{x} \in \R^{\dimpos}, \norm{\vect{x}} \leq 1}$ is the ellipsoid centered at $\elpctr \in \R^{\dimpos}$ and associated with a positive semidefinite matrix\reffn{fn.MatrixSquareRoot} $\elpmat \in \PSDM^{\dimpos}$ and a nonnegative scalar $\elprad \geq 0$, and the stable robot state associated with the reference path point $\refpath(\pathparam)$ is denoted by $\boldsymbol{\refpath}(\pathparam) := (\refpath(\pathparam), 0, \ldots, 0) \in \R^{\order \dimpos}$, and  $\lyapmat \in \PDM^{\order \dimpos}$ is a symmetric positive-definite matrix that uniquely satisfies the Lyapunov equation 
\begin{align}\label{eq.PhDLyapunovEquation}
\tr{(\phdmat_{\phdroots} \otimes \mat{I}_{\dimpos \times \dimpos})} \lyapmat + \lyapmat  (\phdmat_{\phdroots} \otimes \mat{I}_{\dimpos \times \dimpos}) + \tr{\decaymat}\decaymat = 0 
\end{align}
for some (damping/decay) matrix $\decaymat\!\in\! \R^{m \times \order \dimpos}$ such that \mbox{$(\phdmat_{\phdroots} \otimes \mat{I}_{\dimpos \times \dimpos}, \decaymat)$} is observable \cite{khalil_NonlinearSystems2001}.

\addtocounter{footnote}{1}\footnotetext{\label{fn.MatrixSquareRoot}The unique symmetric positive-definite square-root of $\elpmat$ that satisfies $\elpmat = \elpmat^{\frac{1}{2}} \tr{(\elpmat^{\frac{1}{2}})\!}$ can be determined as $\elpmat^{\frac{1}{2}}\!=\!\mat{V}\diag\plist{\sqrt{\sigma_1}, \ldots, \sqrt{\sigma_\order}} \tr{\mat{V}}$ using the singular value decomposition $\elpmat = \mat{V} \diag\plist{\sigma_1, \ldots, \sigma_n} \tr{\mat{V}} $, where $\diag\plist{\sigma_1, \ldots, \sigma_n}$ is the diagonal matrix with  elements ${\sigma_1, \ldots, \sigma_n}$.}

\begin{figure}[t]
\center
\begin{tabular}{@{\hspace{1mm}}c@{\hspace{2mm}}c@{}}
\includegraphics[width = 0.48\columnwidth]{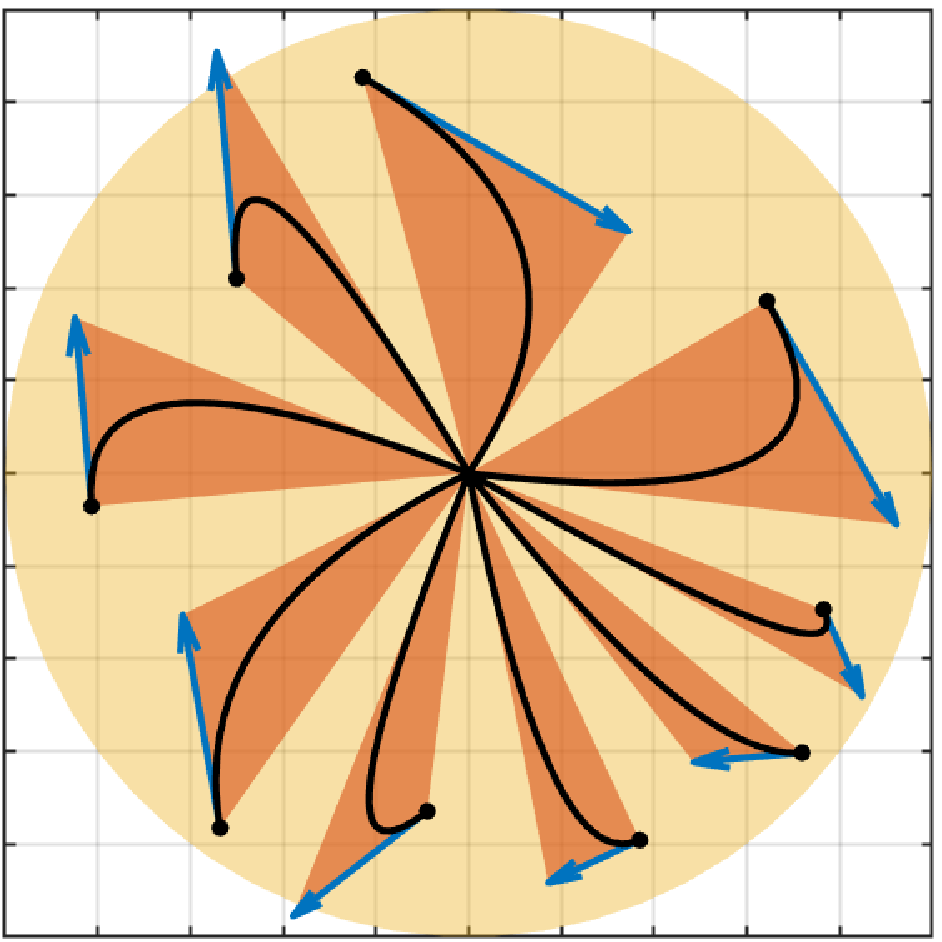} & 
\includegraphics[width = 0.48\columnwidth]{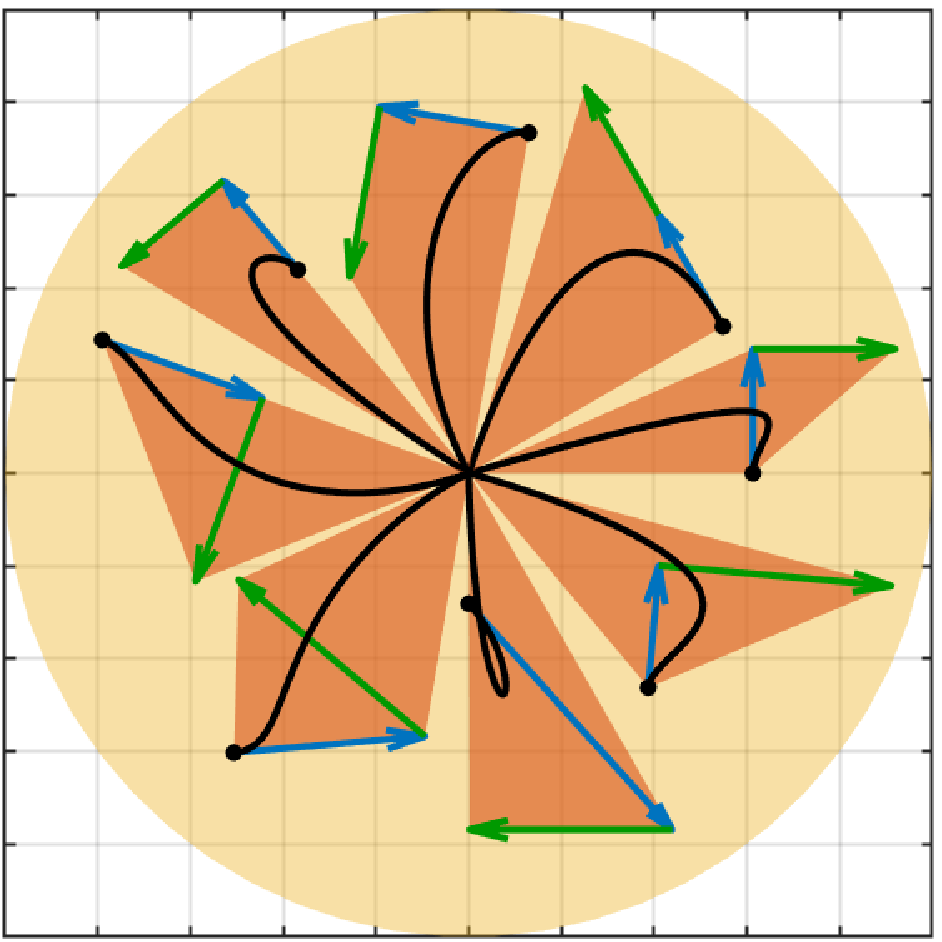}
\end{tabular}
\vspace{-1mm}
\caption{
Lyapunov (yellow circle) and Vandermonde (orange polygons) motion prediction sets for PhD motion control of (left) the second-order and (right) the third-order robot dynamics towards the origin, where the initial robot position is indicated by a black dot and the initial (scaled) velocity (blue) and acceleration (green) are illustrated by arrows.
Here, the characteristic roots $\phdroots$ of PhD control are uniformly spaced over $[-2, -1]$  and $\decaymat=\mat{I}_{\order\dimpos \times \order\dimpos }$.
}
\label{fig.LyapunovEllipsoidVandermondeSimplex}
\vspace{-\baselineskip}
\end{figure}

In addition to defining Lipschitz-continuous collision distance $\dist_{\freespace}(\motionelp_{\pathctrl_{\refpath}}(\state, \pathparam))$ \cite{isleyen_vandewouw_arslan_RAL2022},  the radius of Lyapunov motion ellipsoids is Lipschitz continuous and asymptotically decaying to zero, as required by our time governor framework.  


\begin{lemma}\label{prop.LyapunovEllipsoidRadius}
\emph{(Lyapunov Motion Ellipsoid Radius)}
The radius of the Lyapunov motion ellipsoid $\motionelp_{\phdctrl_{\refpath}}(\state, \pathparam)$ relative to the reference path point $\refpath(\pathparam)$  is given by
\begin{align}
& \setradius_{\refpath(\pathparam)}\plist{\motionelp_{\phdctrl_{\refpath}}(\state, \pathparam)} = \norm{\tr{\mat{I}}_{\order \dimpos \times \dimpos} \lyapmat^{-1} \mat{I}_{\order \dimpos \times \dimpos}}^{\frac{1}{2}} \norm{\state - \boldsymbol{\refpath}(\pathparam)}_{\lyapmat} \nonumber
\end{align}
which is Lipschitz continuous and asymptotically decays to zero along the solution trajectory $\pos(t)$ of $\pos^{(\order)} = \phdctrl_{\refpath}(\state, \pathparam, 0)$.
\end{lemma} 
\begin{proof}
The closed-form expression for the Lyapunov ellipsoid radius follows from the radius of the ellipsoid $\elp(\elpctr, \elpmat, \elprad)$ relative to its center which is given by
\begin{align}
\setradius_{\elpctr}(\elp(\elpctr, \elpmat, \elprad)) = \max_{\norm{\vect{v}} \leq 1} \norm{\elprad \elpmat^{\frac{1}{2}} \vect{v}} = \elprad \norm{\elpmat}^{\frac{1}{2}}.  
\end{align}
The Lipschitz continuity of the Lyapunov ellipsoid radius is a direct consequence of the Lipschitz continuity of the Euclidean norm, due to the reverse triangle inequality, i.e., $\absval{\norm{\vect{a}} - \norm{\vect{b}}} \leq \norm{\vect{a} - \vect{b}}$.
Finally, the asymptotic decay of the Lyapunov ellipsoid radius under the PhD path-following control $\phdctrl_{\refpath}(\state, \pathparam, 0)$ follows from that it is  globally asymptotically stable at the robot state $\state =  \boldsymbol{\refpath}(\pathparam) = (\refpath(\pathparam), 0, \ldots, 0)$.  
\qedhere
\end{proof}

\subsubsection{Vandemonde Motion Prediction for PhD Path Following}

As a more accurate alternative to Lyapunov ellipsoids, we also consider Vandermonde simplexes \cite{arslan_isleyen_2022, isleyen_vandewouw_arslan_RAL2022} for feedback motion prediction of PhD path-following control.
Under the PhD path-following control $\phdctrl_{\refpath}(\state, \pathparam, 0)$ in \refeq{eq.PhDPathFollowingControl}, the robot position trajectory $\pos(t)$, starting from any initial state \mbox{$\state_0=(\pos_{0}^{(0)}, \pos_{0}^{(1)}, \ldots, \pos_{0}^{(\order -1)}) \in \R^{\order \dimpos}$} towards the reference path point $\refpath(\pathparam)$ is contained in the Vandermonde motion simplex $\motionspx_{\phdctrl_{\refpath}}(\state, \pathparam)$ that is defined as \cite{isleyen_vandewouw_arslan_RAL2022}
\begin{align}\nonumber
\motionspx_{\phdctrl_{\refpath}}\!(\state, \pathparam) \!:= \!\conv\!\plist{\!\!\clist{\big.\refpath(\pathparam)} \! \bigcup \clist{\!\sum_{k=0}^{m}\! \frac{\phdgain_{k, \phdroots_{\neg \max}}}{\phdgain_{0, \phdroots_{\neg \max}}} \pos^{(k)}\!\!}_{\! \!0 \leq m \leq \order -1 } } 
\end{align}    
where $\conv$ denotes the convex hull operator, $\phdroots_{\neg \max}$ is the $(\order-1)$-dimensional subvector of PhD characteristic polynomial roots $\phdroots \!=\! \plist{\phdroot_1, \ldots, \phdroot_\order}$ excluding a maximal element that equals to $\max(\phdroots)$, and $\phdgain_{0, \phdroots}, \ldots, \phdgain_{\order, \phdroots} $ are the PhD control gains that satisfies \refeq{eq.PhDRootsGains}. 
As illustrated in \reffig{fig.LyapunovEllipsoidVandermondeSimplex}, Vandermonde simplexes offer significantly more accurate and less conservative feedback motion prediction for PhD path-following control compared to Lyapunov ellipsoids, because Vandermonde simplexes have a stronger dependency on the PhD control gains and the robot state \cite{isleyen_vandewouw_arslan_RAL2022}.

Similar to Lyapunov ellipsoids, Vandermonde motion simplexes induce a Lipschitz-continuous collision distance $\dist_{\freespace}\plist{\motionspx_{\phdctrl_{\refpath}}(\state, \pathparam)}$  \cite{isleyen_vandewouw_arslan_RAL2022} and  have a Lipschitz continuous and asymptotically decaying radius.

\begin{lemma}\label{lem.VandermondeSimplexRadius}
\emph{(Vandermonde Motion Simplex Radius)}
The radius of the Vandemonde motion simplex $\motionspx_{\phdctrl_{\refpath}}(\state, \pathparam)$ relative to reference path point $\refpath(\pathparam)$ is given by
\begin{align} \nonumber
\setradius_{\refpath(\pathparam)}\plist{\motionspx_{\phdctrl_{\refpath}}(\state, \pathparam)\!} = \! \!\max_{0 \leq m \leq \order -1} \left \| \refpath(\pathparam) - \!\! \sum_{k=0}^{m}\! \frac{\phdgain_{k, \phdroots_{\neg\max}}}{\phdgain_{0, \phdroots_{\neg \max}}} \pos^{(k)} \right \|
\end{align}
which is Lipschitz continuous and asymptotically decays to zero along the solution trajectory $\pos(t)$ of $\pos^{(\order)} = \phdctrl_{\refpath}(\state, \pathparam, 0)$.  
\end{lemma} 
\begin{proof}
The closed-form expression for the Vandemonde motion simplex radius follows from its convexity, because the reference path point $\refpath(\pathparam)$ is contained in the convex Vandermonde motion simplex and its radius relative to $\refpath(\pathparam)$ is given by the maximum vertex distance.
The Lipschitz continuity of the Vandermonde simplex radius is due to the Lipschitz continuity of the Euclidean norm and the fact that the maximum of multiple Lipschitz continuous functions are Lipschitz \reffn{fn.MinimumLipschitzContinuity} \cite{hager_JCO1979}.  
Finally, the asymptotic decay property of the Vandermonde simplex radius under the PhD path-following control $\phdctrl_{\refpath}(\state, \pathparam, 0)$ follows from that its globally asymptotically stable  robot state $\state = \boldsymbol{\refpath}(\pathparam)= \plist{\refpath(\pathparam), 0, \ldots, 0}$ satisfies
\begin{equation}
\left \| \refpath(\pathparam) - \!\! \sum_{k=0}^{m}\! \frac{\phdgain_{k, \phdroots_{\neg\max}}}{\phdgain_{0, \phdroots_{\neg \max}}} \pos^{(k)} \right \| = \left \| \sum_{k=1}^{m}\! \frac{\phdgain_{k, \phdroots_{\neg\max}}}{\phdgain_{0, \phdroots_{\neg \max}}} \pos^{(k)} \right \| = 0 \nonumber
\end{equation}
which completes the proof
\qedhere
\end{proof}

\begin{figure*}[t]
\centering
\begin{tabular}{@{}c@{\hspace{1mm}}c@{\hspace{1mm}}c@{\hspace{1mm}}c@{}}
\includegraphics[width=0.24\textwidth]{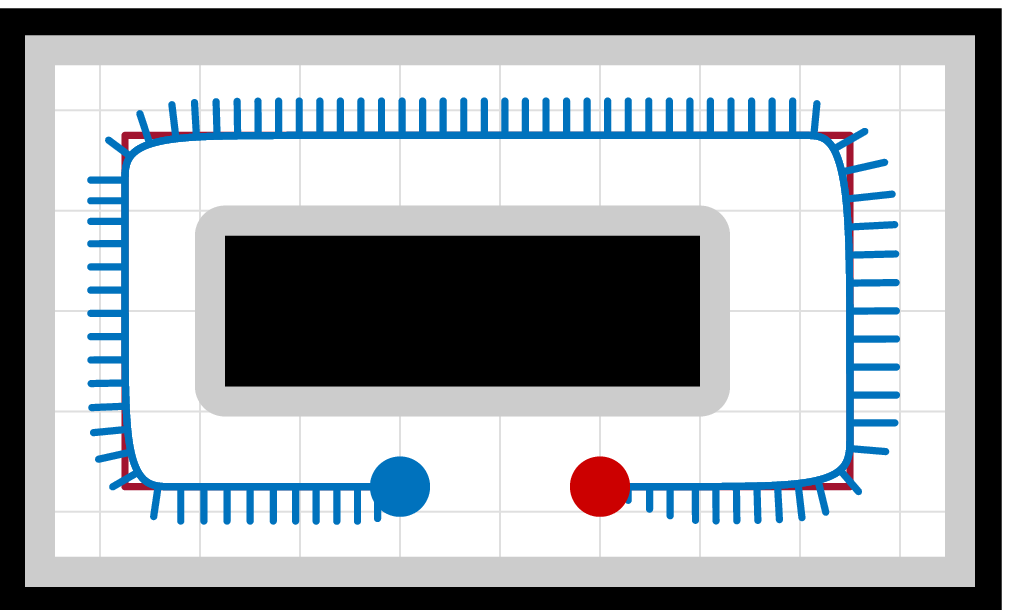} &
\includegraphics[width=0.24\textwidth]{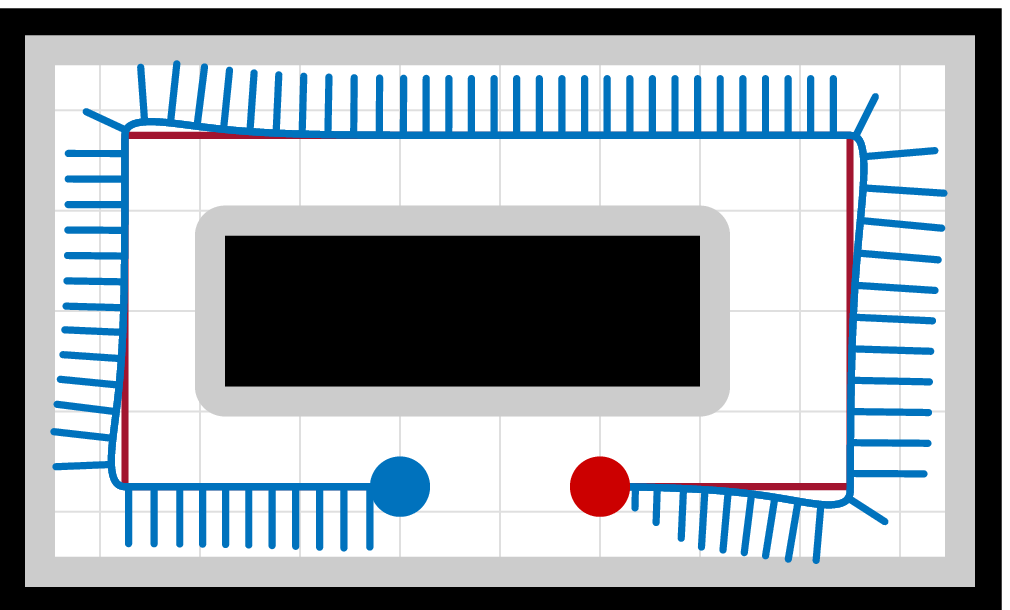} &
\includegraphics[width=0.24\textwidth]{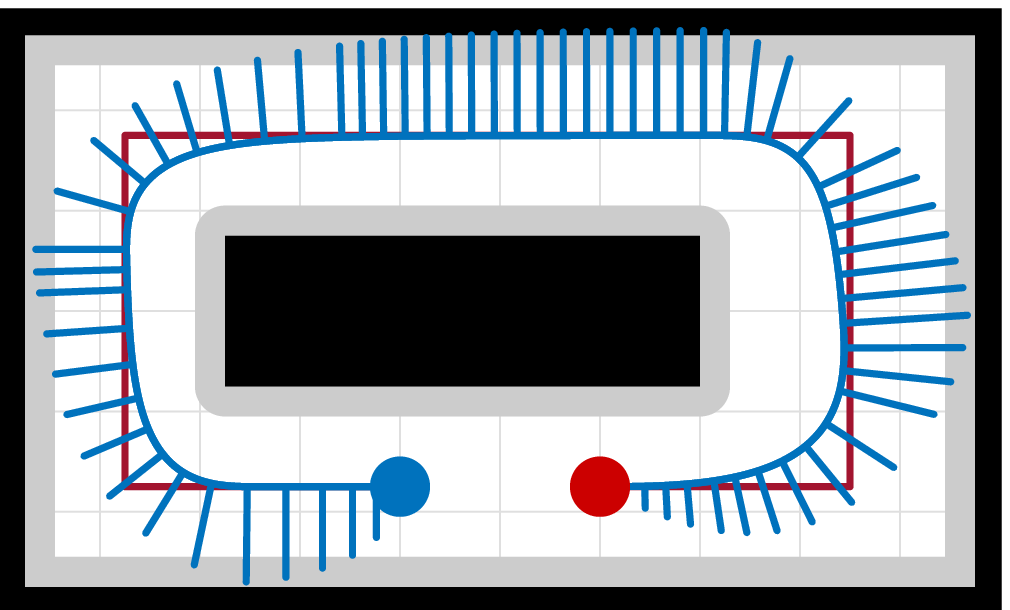}& 
\includegraphics[width=0.24\textwidth]{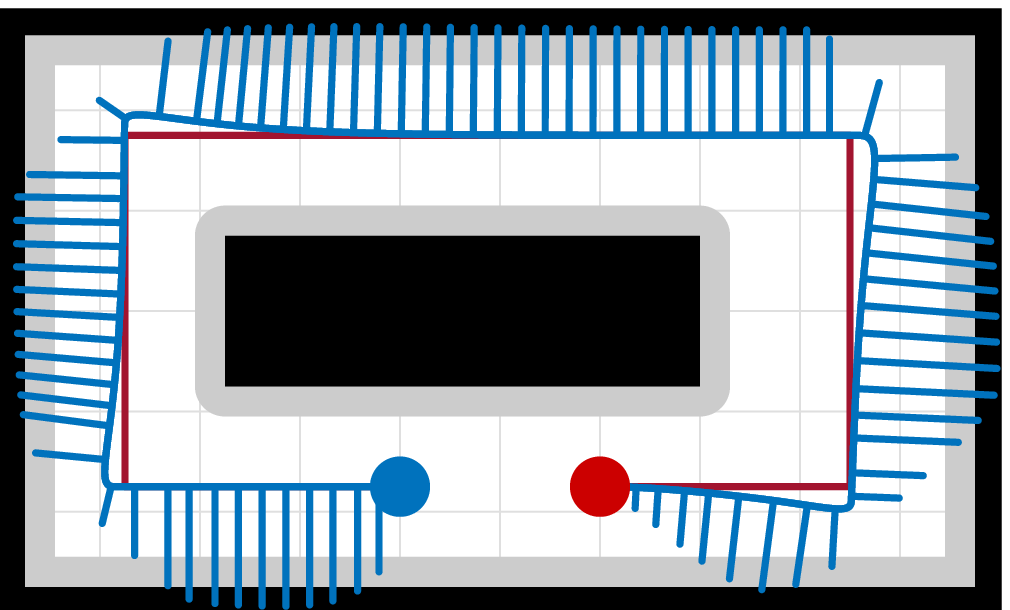} 
\\
\includegraphics[width=0.24\textwidth]{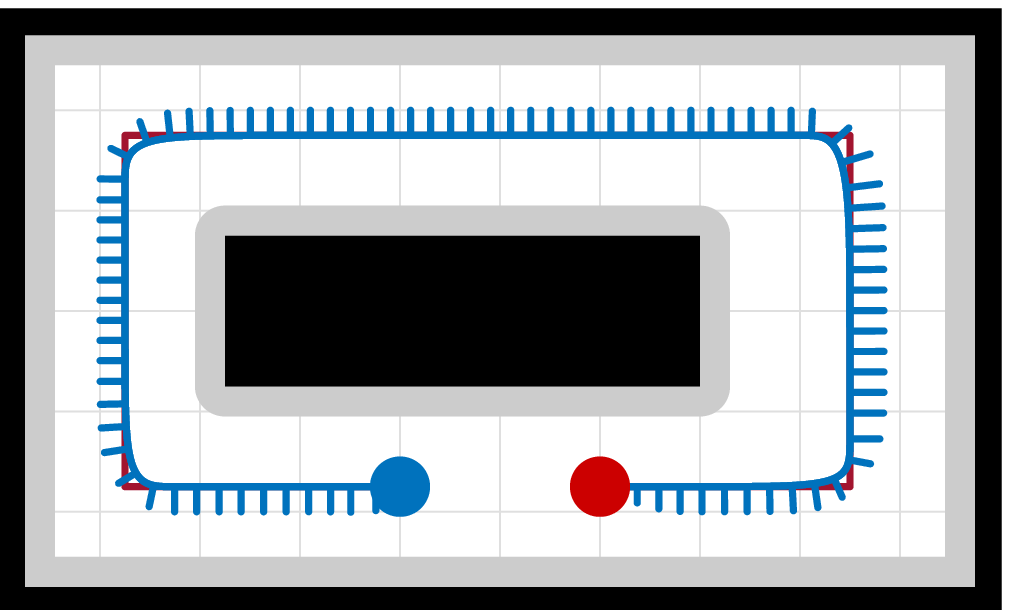} &
\includegraphics[width=0.24\textwidth]{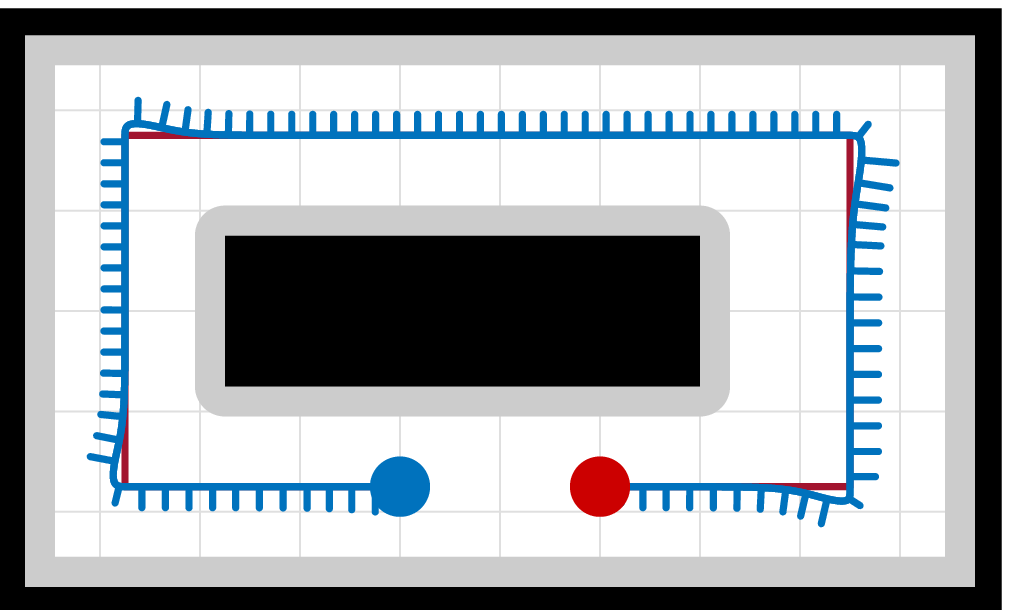} &
\includegraphics[width=0.24\textwidth]{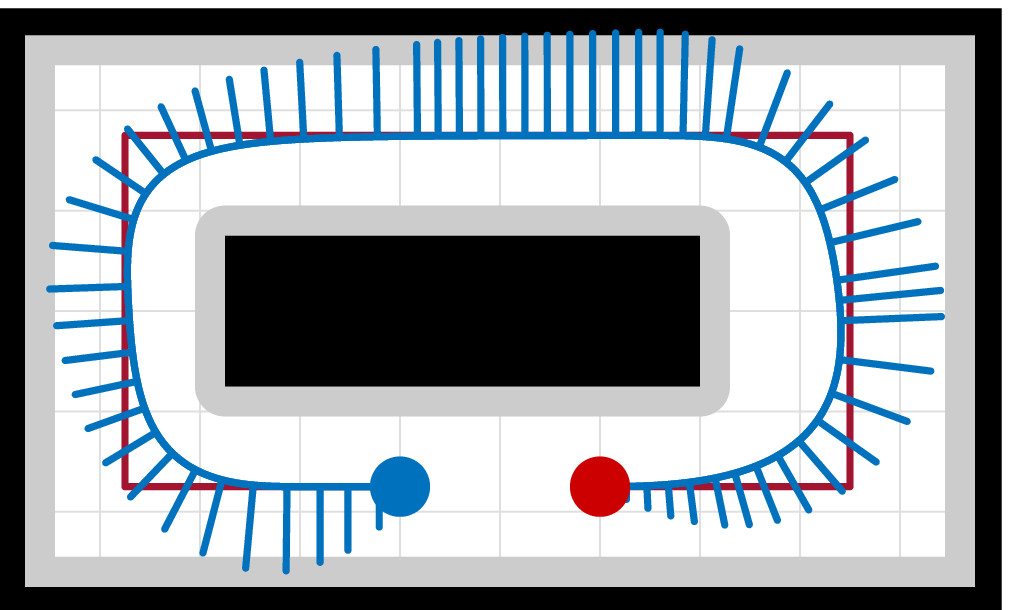}& 
\includegraphics[width=0.24\textwidth]{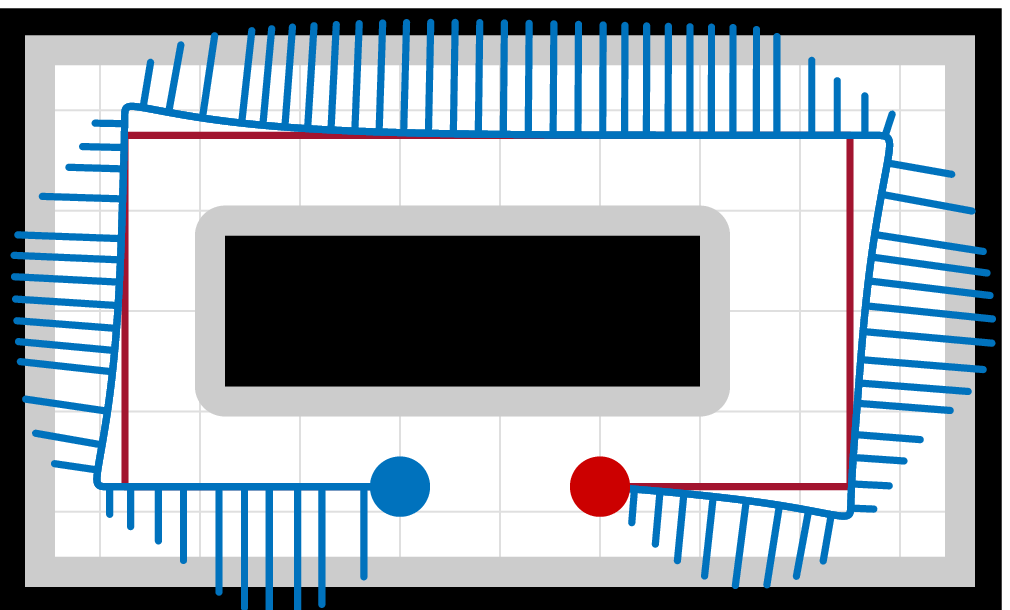} 
\\
\scriptsize{(a)} & \scriptsize{(b)}& \scriptsize{(c)} & \scriptsize{(d)}
\end{tabular}
\caption{Time-governed safe path following in a corridor environment for (top) the second-order, and  (bottom) the third-order robot model using (a,b) Lyapunov motion ellipsoids and (c,d) Vandermonde motion simplexes based on (a, c) reference-position-only feedback and (b, d) reference-position-and-velocity feedback. The robot trajectory (blue line) and the reference path (red line) are in clockwise direction starting at the blue circle (robot body) and ending at the red circle, and the robot velocity bars (blue) are placed along the robot trajectory.}
\label{fig.SafePathFollowingCorridor}
\end{figure*}

\begin{figure}[t]
\centering
\begin{tabular}{@{}c@{\hspace{1mm}}c@{}}
\includegraphics[width=0.49\columnwidth]{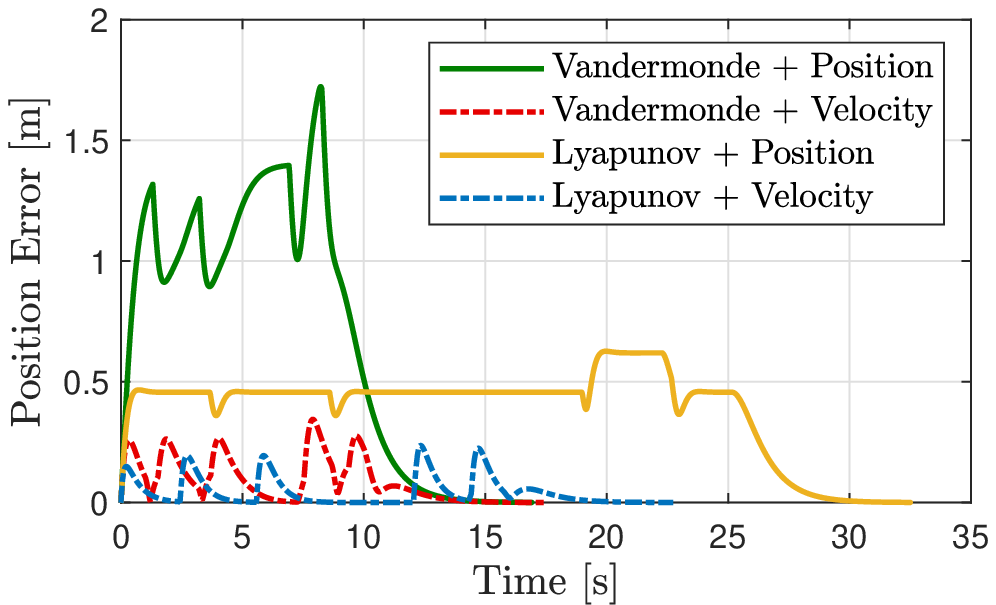} &
\includegraphics[width=0.49\columnwidth]{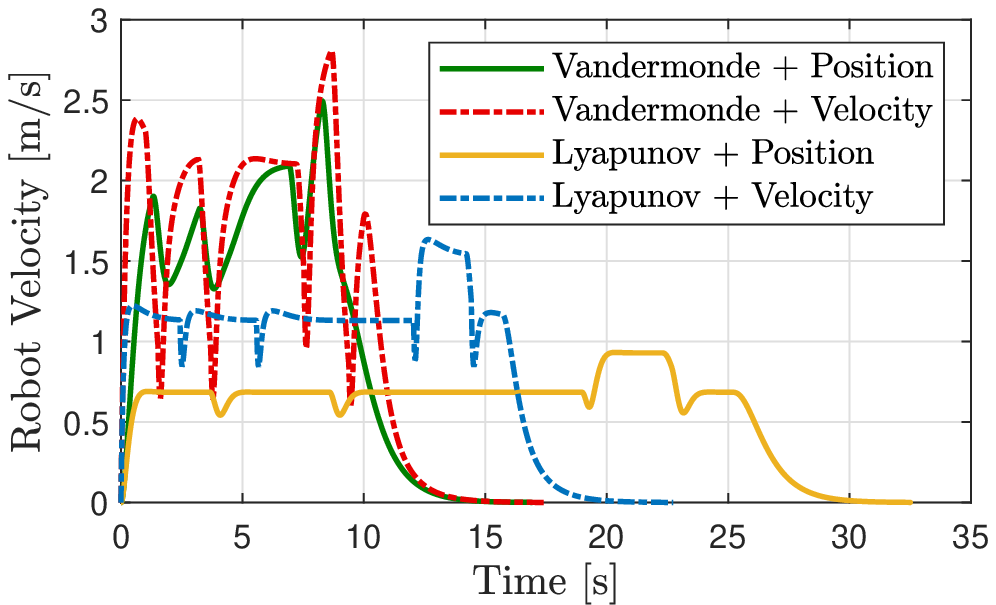} 
\\ 
\includegraphics[width=0.49\columnwidth]{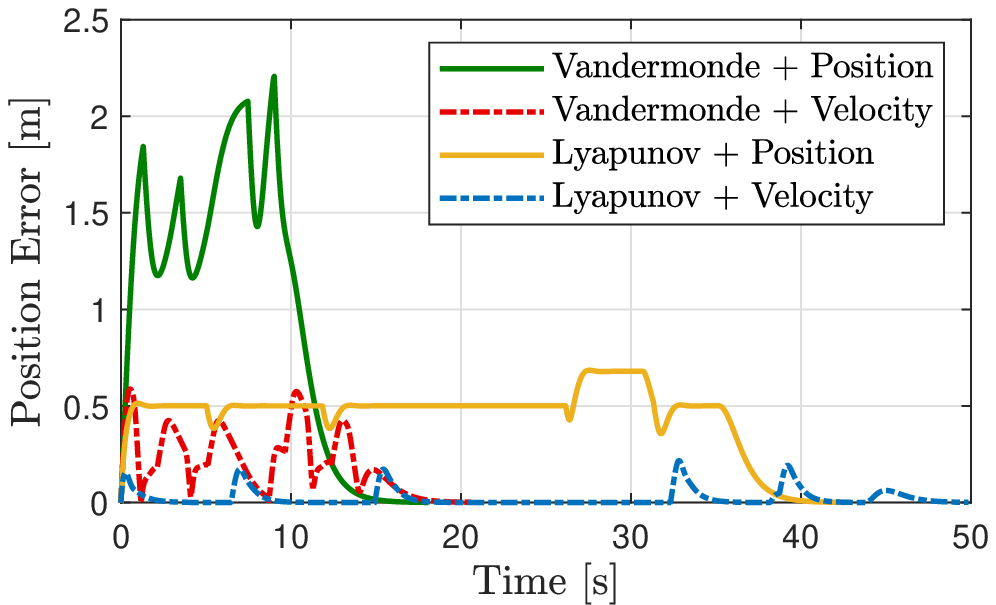} &
\includegraphics[width=0.49\columnwidth]{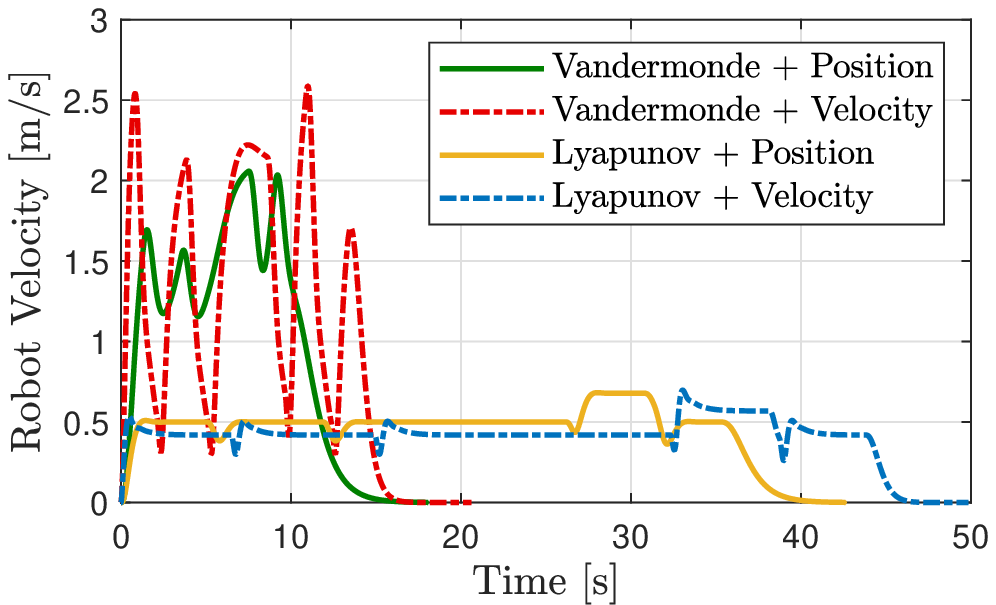} 
\end{tabular}
\caption{Performance of time-governed safe path-following control in a corridor environment for (top) the second-order  and (bottom) the third-order robot dynamics, where the path-following performance is measured in terms of  (left) position error $\norm{\pos(t) - \refpath(\pathparam(t))}$ and (right) robot velocity $\norm{\dot{\pos}(t)}$.}
\label{fig.PathFollowingPerformanceCorridor}
\end{figure}

\begin{figure*}[t]
\centering
\begin{tabular}{@{}c@{\hspace{1mm}}c@{\hspace{1mm}}c@{\hspace{1mm}}c@{}}
\includegraphics[width=0.24\textwidth]{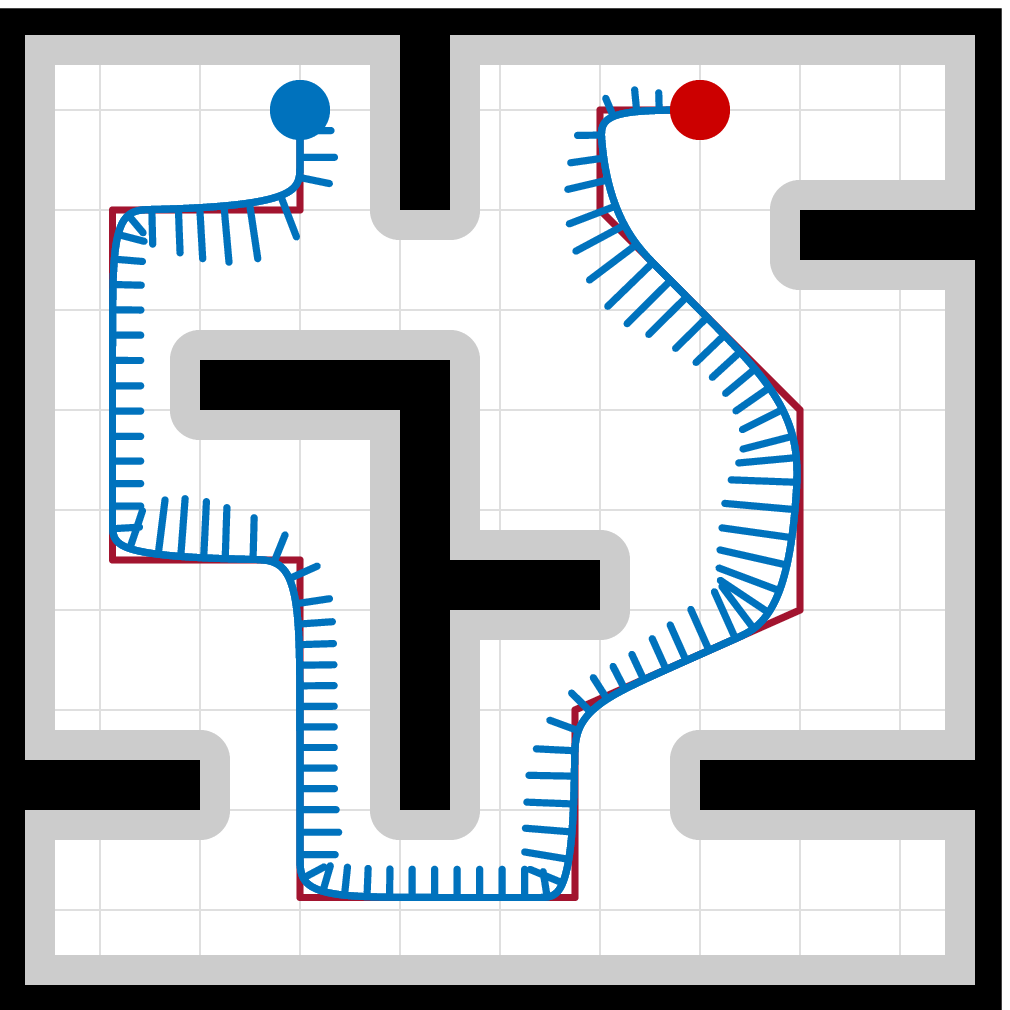} &
\includegraphics[width=0.24\textwidth]{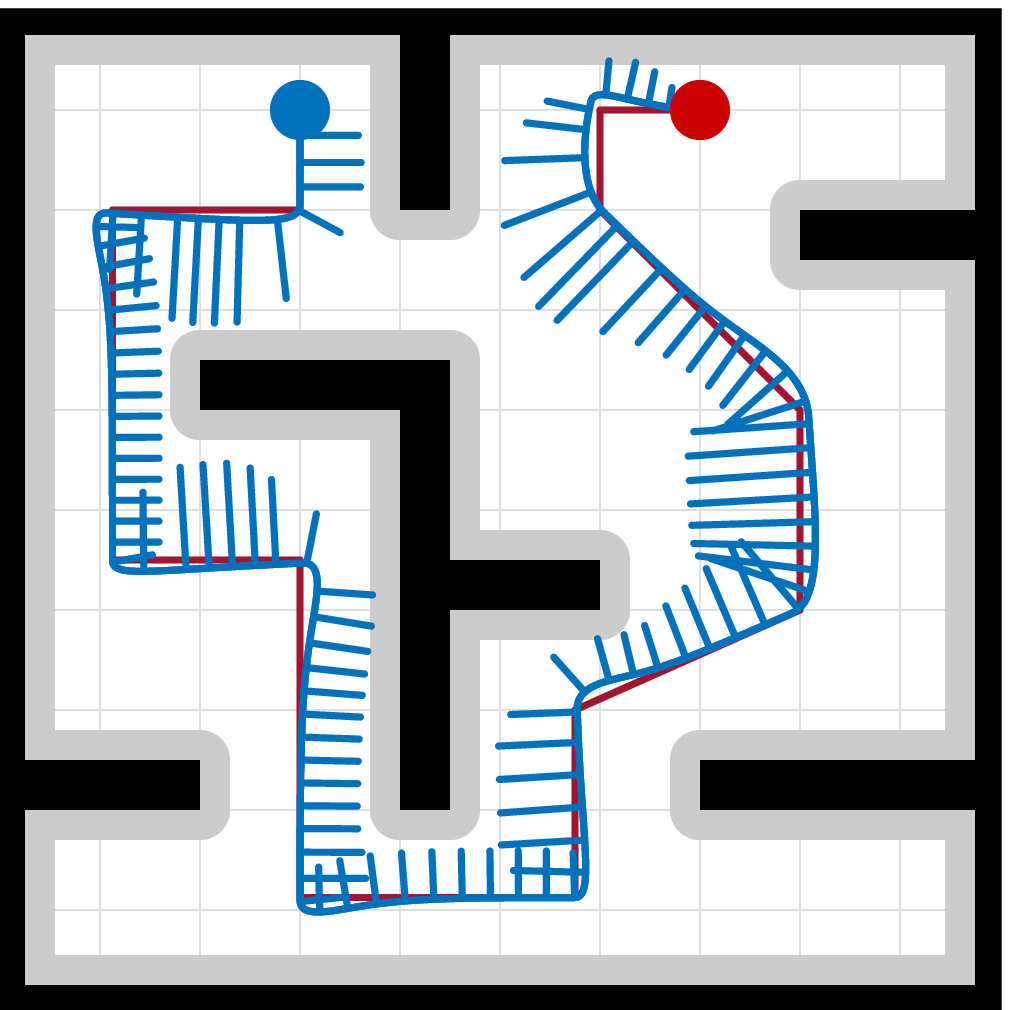} &
\includegraphics[width=0.24\textwidth]{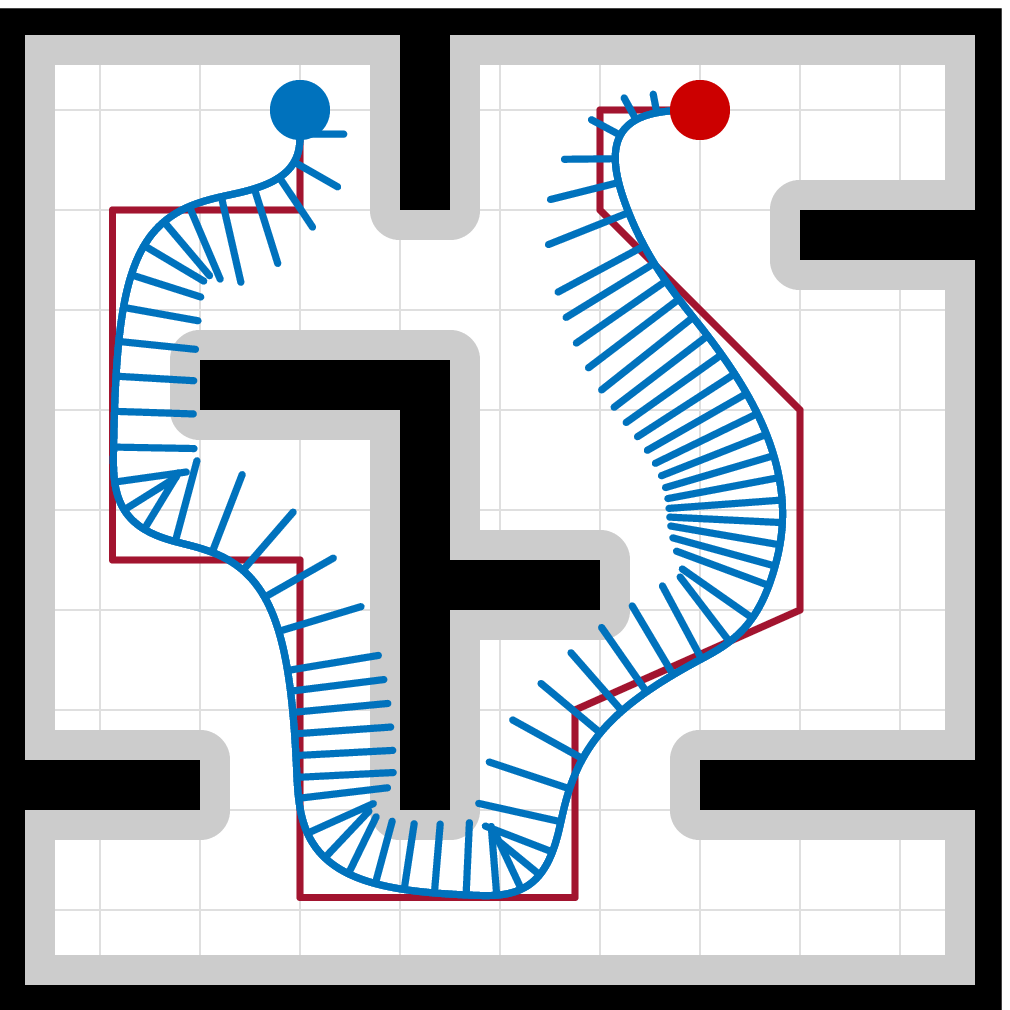}& 
\includegraphics[width=0.24\textwidth]{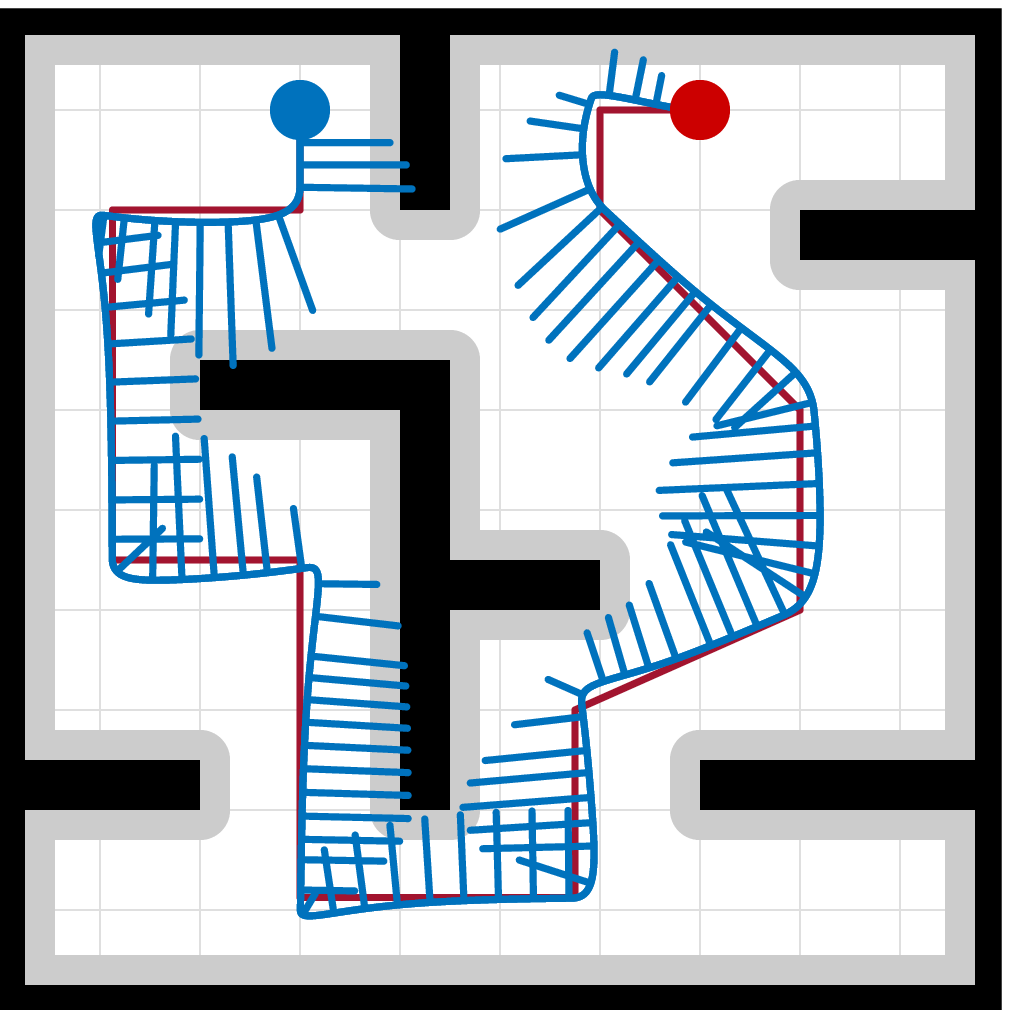} 
\\
\includegraphics[width=0.24\textwidth]{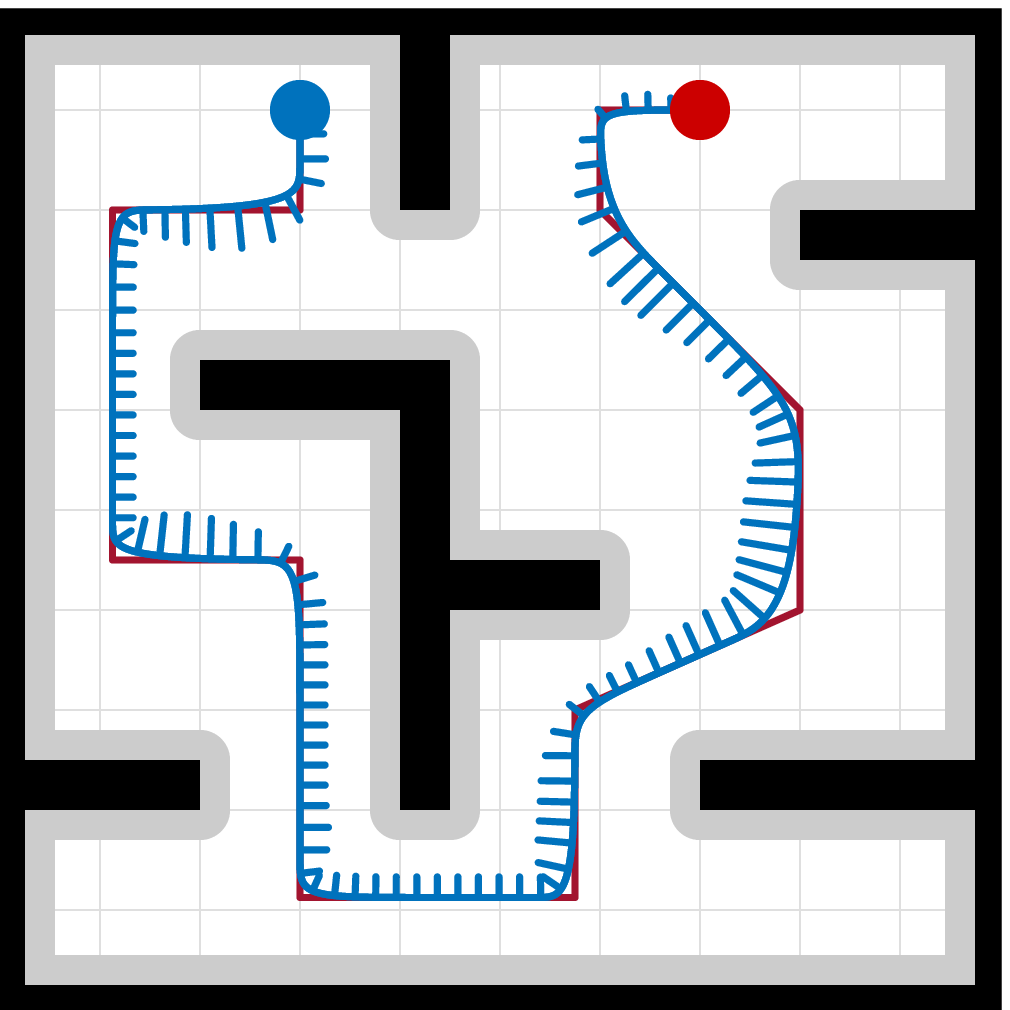} &
\includegraphics[width=0.24\textwidth]{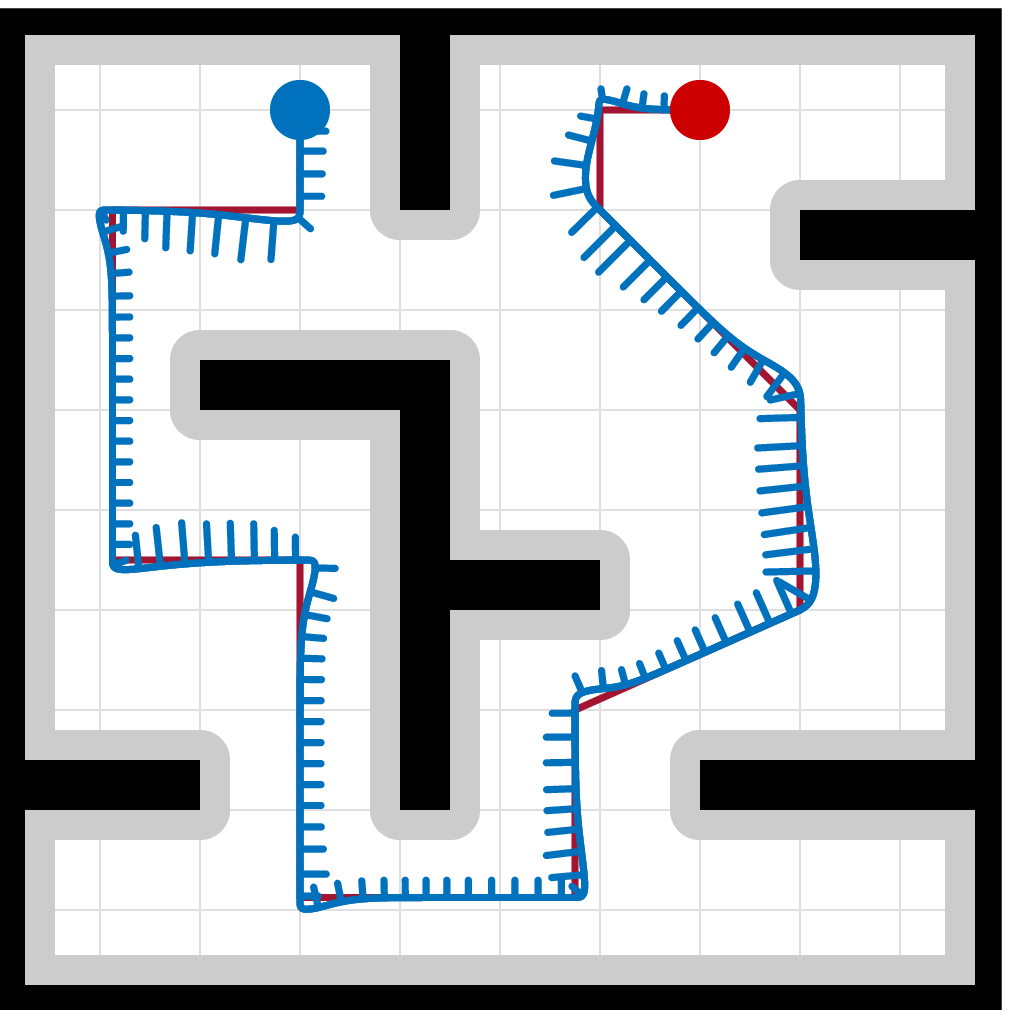} &
\includegraphics[width=0.24\textwidth]{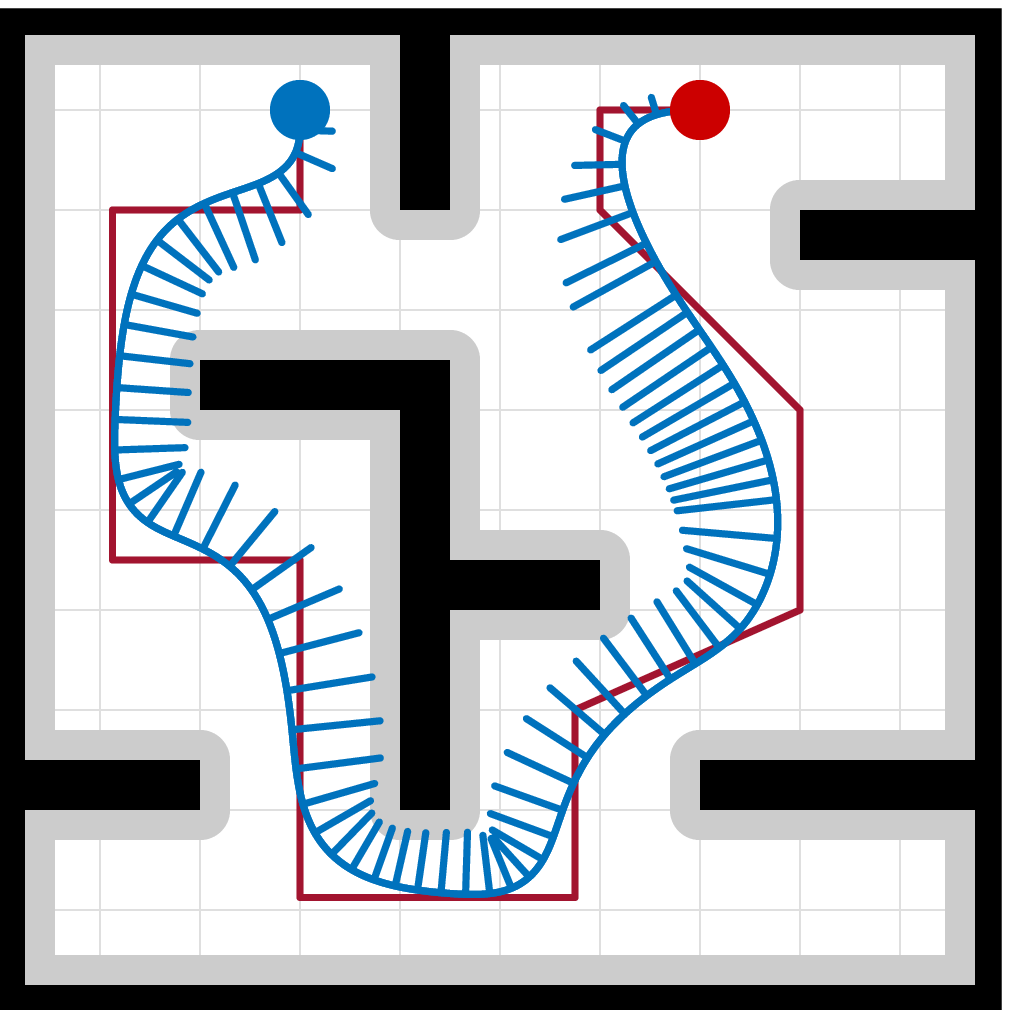}& 
\includegraphics[width=0.24\textwidth]{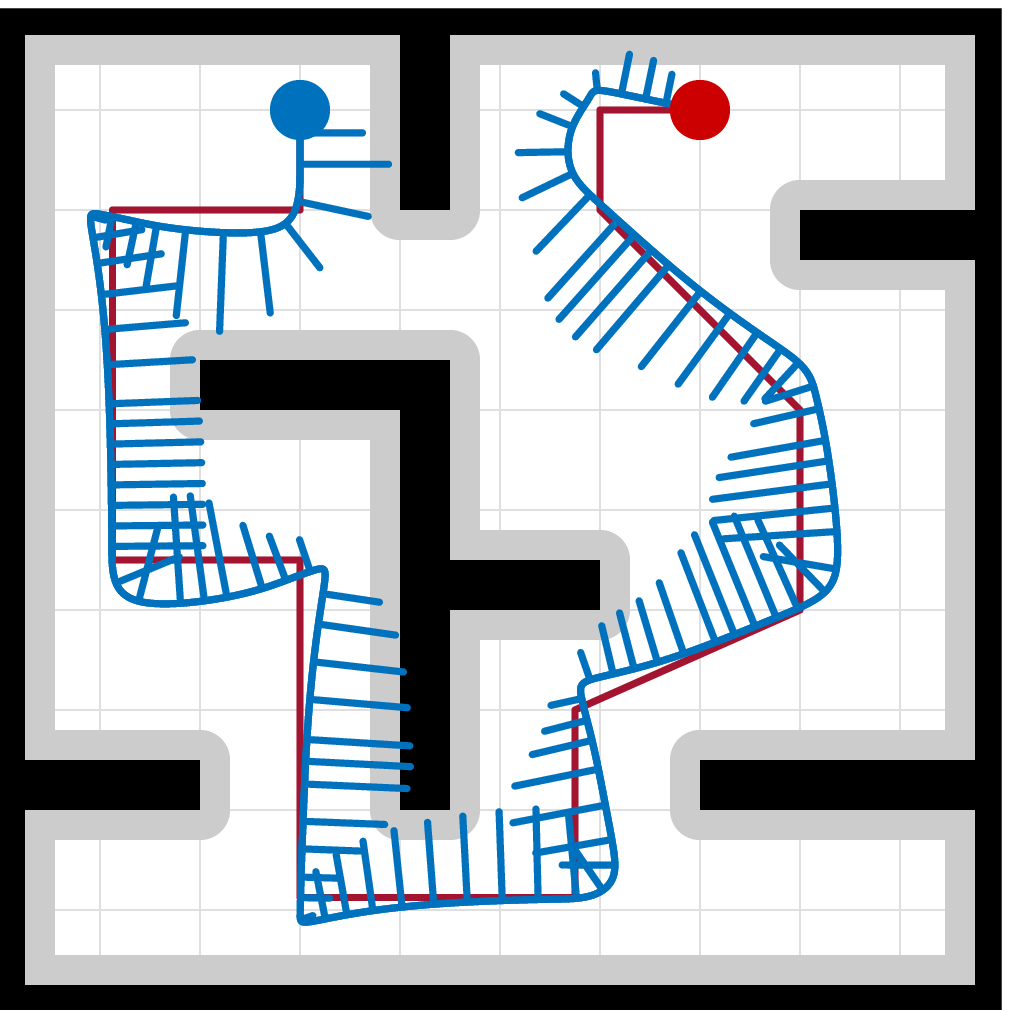} 
\\
\scriptsize{(a)} & \scriptsize{(b)}& \scriptsize{(c)} & \scriptsize{(d)}
\end{tabular}
\caption{Time-governed safe path following in an office environment for (top) the second-order, and  (bottom) the third-order robot model using (a,b) Lyapunov motion ellipsoids and (c,d) Vandermonde motion simplexes based on (a, c) reference-position-only feedback and (b, d) reference-position-and-velocity feedback. The robot trajectory (blue line) and the reference path (red line) are in counter-clockwise direction starting at the blue circle (robot body) and ending at the red circle, and the robot velocity bars (blue) are placed along the robot trajectory.}
\label{fig.SafePathFollowingOffice}
\end{figure*}

\begin{figure}[t]
\centering
\begin{tabular}{@{}c@{\hspace{1mm}}c@{}}
\includegraphics[width=0.49\columnwidth]{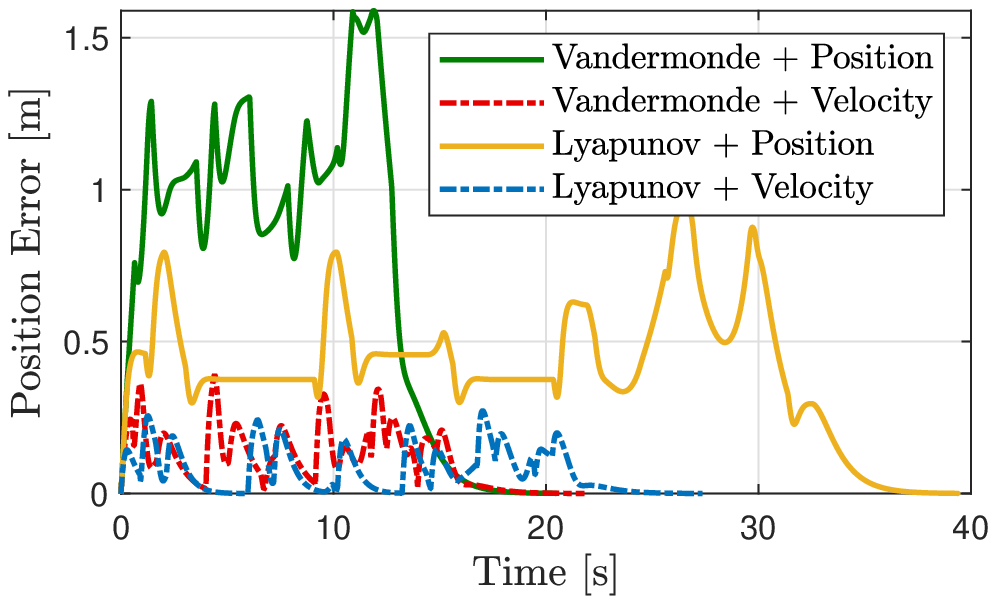} & 
\includegraphics[width=0.49\columnwidth]{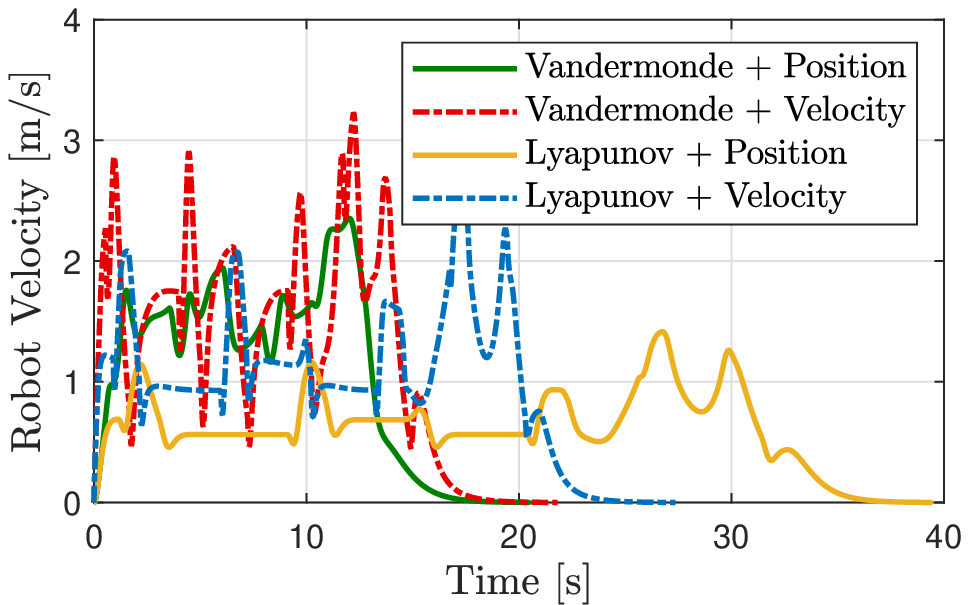} 
\\
\includegraphics[width=0.49\columnwidth]{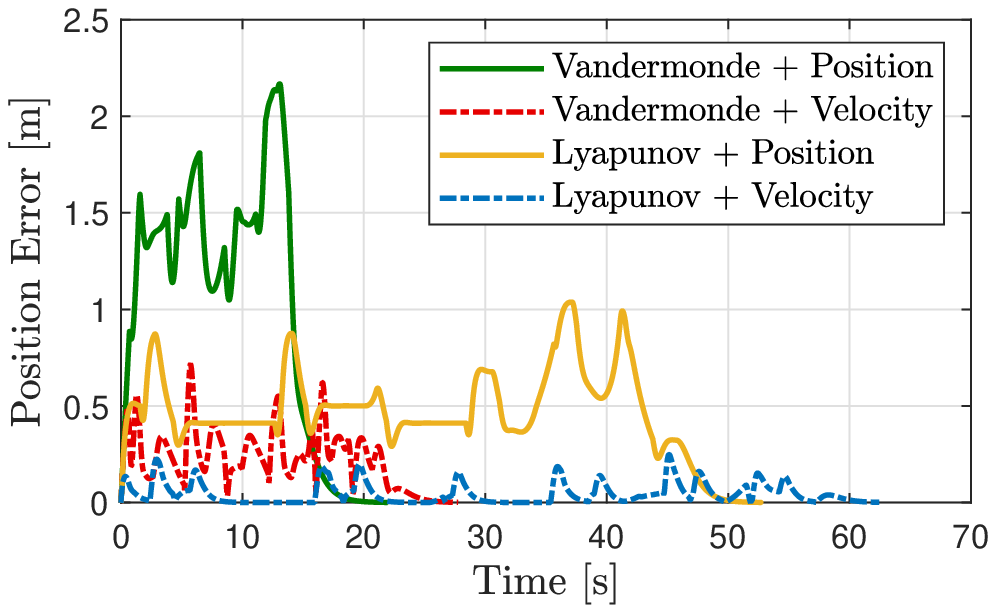} & 
\includegraphics[width=0.49\columnwidth]{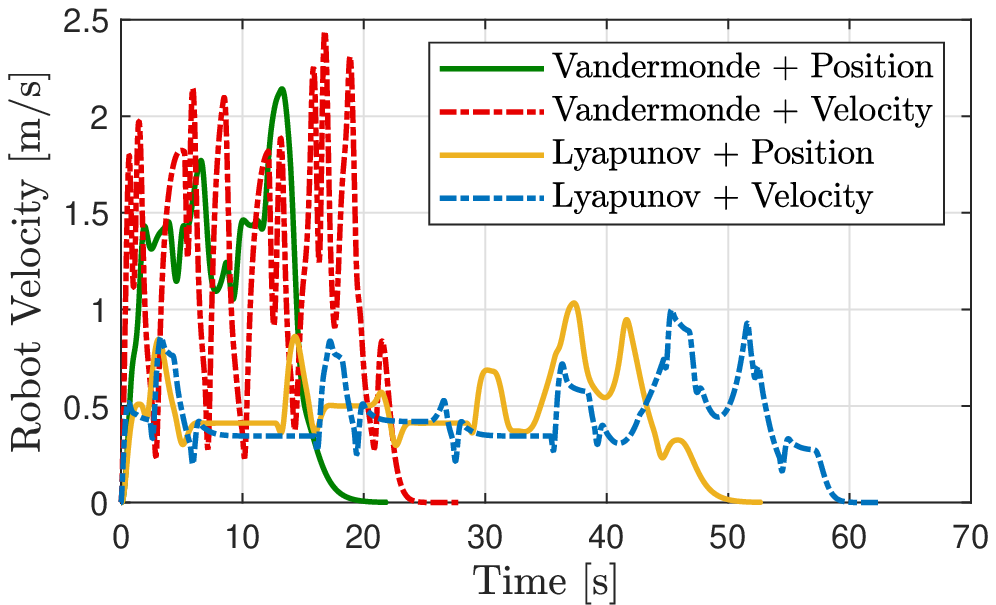} 
\end{tabular}
\caption{Performance of time-governed safe path following in an office environment for (top) the second-order  and (bottom) the third-order robot dynamics, where the path-following performance is measured in terms of  (left) position error $\norm{\pos(t) - \refpath(\pathparam(t))}$ and (right) robot velocity $\norm{\dot{\pos}(t)}$.}
\label{fig.PathFollowingPerformanceOffice}
\end{figure}

\section{Numerical Simulations}
\label{sec.NumericalSimulations}

In this section, we provide numerical simulations%
\footnote{\label{fn.Simulations} For all simulations, the characteristic polynomial roots of the PhD path-following control in \refeq{eq.PhDPathFollowingControl} are  set equally as $\phdroots = (-3.0, \ldots, -3.0)$, and we set the control coefficients for the time governor in \refeq{eq.SafeTimeGovernor} as $\gain_{\safetylevel} = 3.0$ and $\gain_{\pathparam} = 1.0$, and we use $\decaymat = \mat{I}_{\order \dimpos \times \order \dimpos}$ in \refeq{eq.PhDLyapunovEquation} for the Lyapunov motion prediction.
All simulations are obtained by numerically solving the time-governed path-following control dynamics using the \texttt{ode45} function of MATLAB.
We use the arc-length parametrization of a given reference path $\refpath(s)$ such that the reference path length $L$  determines the path parameter range as $[\minpathparam, \maxpathparam] = [0, L]$. 
Please see the accompanying video for the animated path-following motion.
} 
to demonstrate safe PhD path-following control of the second-order and the third-order robot models using Lyapunov and Vandermonde motion prediction in corridor-like and office-like environments.
We investigate the role of reference position and velocity feedback as well as feedback motion prediction on  the path-following performance and robot motion.
As a performance measure for path-following control, we consider the Euclidean distance $\norm{\pos(t)\! -\! \refpath(\pathparam(t))}$ between the robot position $\pos(t)$ and the reference path point $\refpath(\pathparam(t))$, and the robot velocity $\norm{\dot{\pos}(t)}$ which is related to the  travel time. 

\subsection{Safe Path-Following Control in a Corridor Environment}

As a first example, shown in \reffig{fig.SafePathFollowingCorridor}, we consider safe path-following control in a corridor environment, because safe and fast robot motion control in such tight spaces is challenging for highly dynamic robotic systems \cite{li_arslan_atanasov_ICRA2020}.
The provable safety and convergence properties (\refprop{prop.Safety} and \refprop{prop.Convergence}) of our time governed path-following control approach ensure that the robot successfully completes the path-following task without any collision with obstacles, irrespective of the order of robot dynamics, path-following control, and feedback motion prediction, but the resulting robot motion significantly differs in terms of path-following performance, robot speed, and so travel time.
As seen in \reffig{fig.SafePathFollowingCorridor} and \reffig{fig.PathFollowingPerformanceCorridor}, the more accurate Vandermonde motion prediction always results in faster robot motion and shorter travel time than the more conservative Lyapunov motion prediction.
Because symmetric Lyapunov ellipsoids always keep the robot cautious about sideways collisions with corridor walls whether it is relevant (e.g., when approaching to the end of a corridor) or not (e.g., while moving along the corridor).
Moreover, we observe in  \reffig{fig.SafePathFollowingCorridor} and \reffig{fig.PathFollowingPerformanceCorridor} that, compared to the PhD path-following control $\phdctrl_{\refpath}(\state, \pathparam,0)$ with reference-position-only feedback, the PhD path-following control $\phdctrl_{\refpath}(\state, \pathparam, \dot{\pathparam})$ with reference-position-and-velocity feedback improves the path-following performance by achieving a lower distance between the robot and the reference path in average, which mitigates the corner cutting problem, but results in overshooting around sharp corners.
Also, since feedback motion prediction becomes naturally less accurate with increasing system order, we see that robot motion gets slower for high-order robot dynamics, which might be resolved by using higher-order-reference-derivative feedback and feedforward control, i.e., path-tracking control.

\subsection{Safe Path-Following Control in an Office Environment}

To demonstrate how feedback motion prediction enables autonomously adapting robot motion around arbitrarily shaped and placed obstacles, we consider safe path-following control in an office-like cluttered environment, shown in \reffig{fig.SafePathFollowingOffice}.
We observe in \reffig{fig.SafePathFollowingOffice} that the robot slows down when (the intended motion of) the robot becomes closer to obstacles; and it speeds up when there is a large open space (in the front of the robot). 
Hence, as seen in \reffig{fig.SafePathFollowingOffice} and \reffig{fig.PathFollowingPerformanceOffice}, asymmetric Vandermonde motion simplexes show a superior performance in adapting robot motion around obstacles because  symmetric Lyapunov motion ellipsoids are less accurate and capable of relating the geometry of the robot motion to the geometry of the environment.
Moreover, compared to the reference-position-only feedback control $\phdctrl_{\refpath}(\state, \pathparam, 0)$,  the PhD path-following control $\phdctrl_{\refpath}(\state, \pathparam, \dot{\pathparam})$ with reference-position-and-velocity feedback  yields better path-following performance with a lower gap between the robot and reference path, but this causes less smooth and sharper turns if the reference path has sharp turns.
Overall, safe PhD path-following control with Vandermonde motion prediction and reference-position-only feedback yields faster and smoother robot motion.

\section{Conclusions}
\label{sec.Conclusions}

In this paper, we present a provably correct feedback path time-parametrization approach, called \emph{time governors}, for safe path-following control using feedback motion prediction.
We describe a generic time governor design approach, with essential building blocks and their technical requirements, which allows systematically integrating high-level path planning and low-level path-following control methods with safety and convergence guarantees. 
We provide an example application of time governors for PhD path-following control with Lyapunov and Vandemonde motion predictions. 
In numerical simulations, we demonstrate the effectiveness of time governors for safe PhD path-following control around obstacles.
We conclude that accurate feedback motion prediction relates the geometry of robot motion to the geometry of environment more effectively and so is essential for fast robot motion. 
Moreover, path-following performance can be further improved by closing the gap between path-following and trajectory-tracking approaches using higher-order-reference-derivative feedback.

Our current work on progress focuses on the design of time governors for safe path-following control of nonholonomically constrained robots such as differential drive vehicles \cite{isleyen_vandewouw_arslan_arXiv2022} for sensor-based mobile robot navigation in dynamic and unknown environments \cite{arslan_koditschek_IJRR2019}.  
Another promising research direction is the design of higher-order observer-based time governors for safe, fast, and  precise path-following control.





\bibliographystyle{IEEEtran}
\bibliography{references}

\vfill

\end{document}